\documentclass[twoside]{article}

\usepackage[preprint]{aistats2023}
 
\usepackage{enumitem}
\usepackage[utf8]{inputenc}
\usepackage{graphicx}
\usepackage{subfig}
\usepackage{booktabs} 
\usepackage{amsmath, amsfonts, amsthm} 
\usepackage{mynotation}
\usepackage[hidelinks]{hyperref}   

\makeatletter
\newcommand{\longdash}[1][2em]{%
  \makebox[#1]{$\m@th\smash-\mkern-7mu\cleaders\hbox{$\mkern-2mu\smash-\mkern-2mu$}\hfill\mkern-7mu\smash-$}}
\makeatother
\newcommand{\omitskip}{\kern-\arraycolsep}
\newcommand{\llongdash}[1][2em]{\longdash[#1]\omitskip}
\newcommand{\rlongdash}[1][2em]{\omitskip\longdash[#1]}

\DeclareUnicodeCharacter{2009}{\,} 
\usepackage[style=ieee,natbib=true,backend=biber,doi=false,isbn=false,url=false,eprint=false,mincitenames=1,maxcitenames=1,uniquelist=false,alldates=year, maxbibnames=10]{biblatex}
\AtEveryBibitem{\clearlist{language}\clearname{editor}\clearfield{note}}
\bibliography{refs}

\newtheorem{theorem}{Theorem}
\newtheorem*{theorem*}{Theorem}

\newtheorem{proposition}[theorem]{Proposition}
\newtheorem{lemma}[theorem]{Lemma}
\newtheorem*{definition*}{Definition}

\begin{document}

\twocolumn[

\aistatstitle{Surprises in adversarially-trained linear regression}

\aistatsauthor{Ant\^onio H. Ribeiro \And   Dave Zachariah \And Thomas B. Sch\"on }

\aistatsaddress{Uppsala University \And Uppsala University \And Uppsala University } ]

\begin{abstract}
    State-of-the-art machine learning models can be vulnerable to very small input perturbations that are adversarially constructed. Adversarial training is an effective approach to defend against such examples. It is formulated as a min-max problem, searching for the best solution when the training data was corrupted by the worst-case attacks. For linear regression problems, adversarial training can be formulated as a convex problem. We use this reformulation to make two technical contributions: First, we formulate the training problem as an instance of robust regression to reveal its connection to parameter-shrinking methods, specifically that $\ell_\infty$-adversarial training produces sparse solutions. Secondly, we study adversarial training in the overparameterized regime, i.e. when there are more parameters than data. We prove that adversarial training with small disturbances gives the solution with the minimum-norm that interpolates the training data. Ridge regression and lasso approximate such interpolating solutions as their regularization parameter vanishes. By contrast, for adversarial training, the transition into the interpolation regime is abrupt and for non-zero values of disturbance. This result is proved and illustrated with numerical examples.
\end{abstract}

\section{Introduction}

Robustness is a fundamental goal in machine learning that often surpasses that of achieving high performance in a controlled scenario. Indeed, the brittleness of modern machine learning models presents a challenge to their deployment in critical situations.
The framework of adversarial attacks has generated striking  examples of their vulnerability, where very small input perturbations can cause a substantial performance drop in otherwise state-of-the-art models, see for instance~\citep{bruna_intriguing_2014,goodfellow_explaining_2015, kurakin_adversarial_2018, fawzi_analysis_2018, ilyas_adversarial_2019, yuan_adversarial_2019}.
It considers inputs contaminated with disturbances deliberately chosen to maximize the model error. 

While a lot of current research focuses on adversarial examples for deep learning models, there is a growing body of work~\cite{taheri_asymptotic_2021,javanmard_precise_2020,hassani_curse_2022,min_curious_2021,yin_rademacher_2019} that study fundamental properties of adversarial attacks in simpler, linear models.
There is a sound reason for this focus: linear models allow for analytical analysis while still reproducing phenomena of interest. Indeed, while there was initial speculation that the highly nonlinear nature of deep neural networks was the cause of its vulnerabilities to adversarial  attacks~\citep{bruna_intriguing_2014} the idea was later dismissed and the vulnerabilities can be observed in purely linear settings~\citep{goodfellow_explaining_2015,ribeiro_overparameterized_2022}.

Adversarial training~\citep{madry_deep_2018} is one of the most effective approaches for deep learning models to defend against adversarial attacks. The topic has been extensively studied by the research community, see e.g.~\citep{huang_learning_2016, madry_deep_2018, bai_recent_2021}. The idea is straightforward: It considers the situation where the model is trained training on samples that have been modified by an adversary. The hope is that the model will be more robust when faced with new adversarially perturbed samples. 

Mathematically, adversarial training is formulated as a min-max problem, searching for the best solution to the worst-case attacks.  In the case of linear regression, consider a training dataset $\{(\x_i, y_i)\}_{i=1}^{n}$ consisting of $n$ datapoints of dimension $\R^m \times \R$. Let $\param$ denote the unknown parameters. Adversarial training now corresponds to solving the optimization problem
\begin{equation}
\label{eq:AdvTrainOrig}
     \min_{\param}~\frac{1}{n}\sum_{i=1}^n{\max_{\|\dx_i \|_\infty \le \delta}(y_i - (\x_i + \dx_i)^\trnsp\param )^2}
  \end{equation}

\begin{figure*}
    \vspace{-10pt}
    \centering
    \subfloat[lasso regression]{\includegraphics[width=0.42\textwidth]{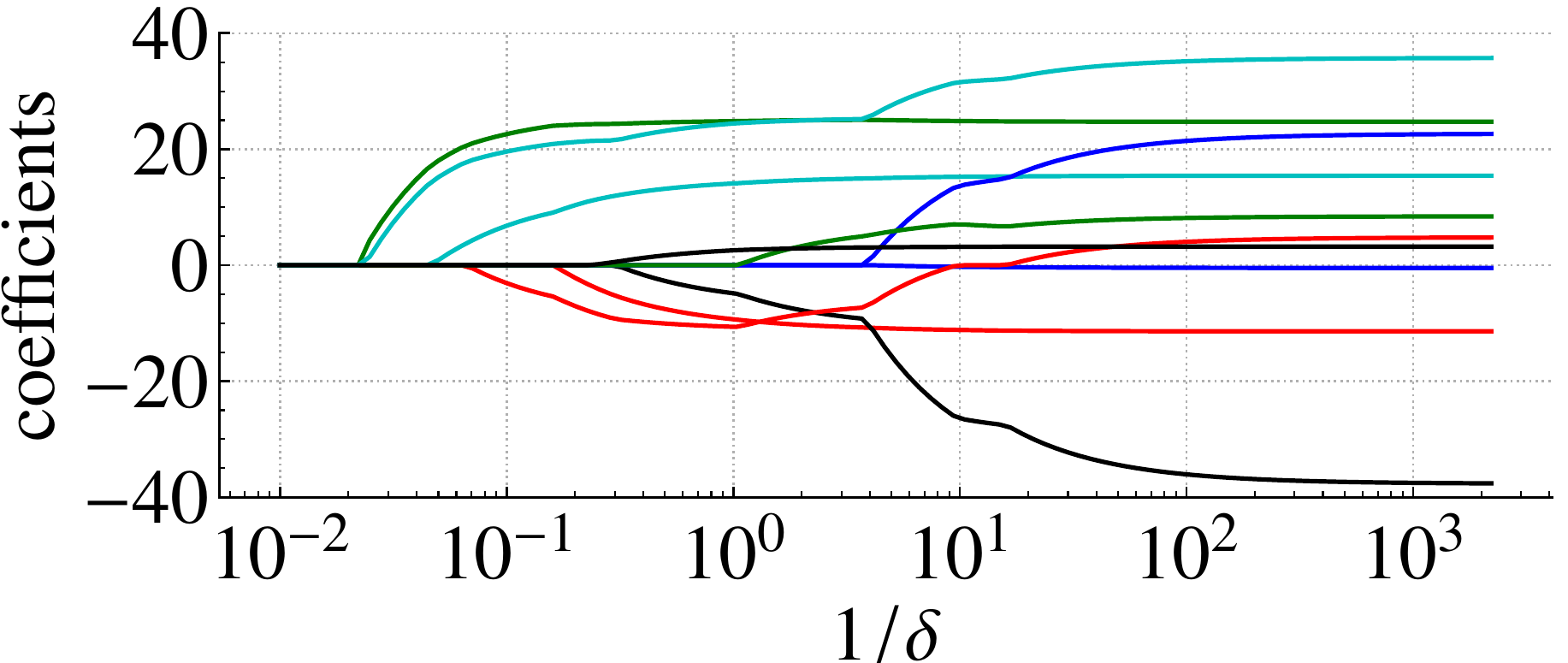}}
    \subfloat[$\ell_\infty$-adversarial training  ]{\includegraphics[width=0.42\textwidth]{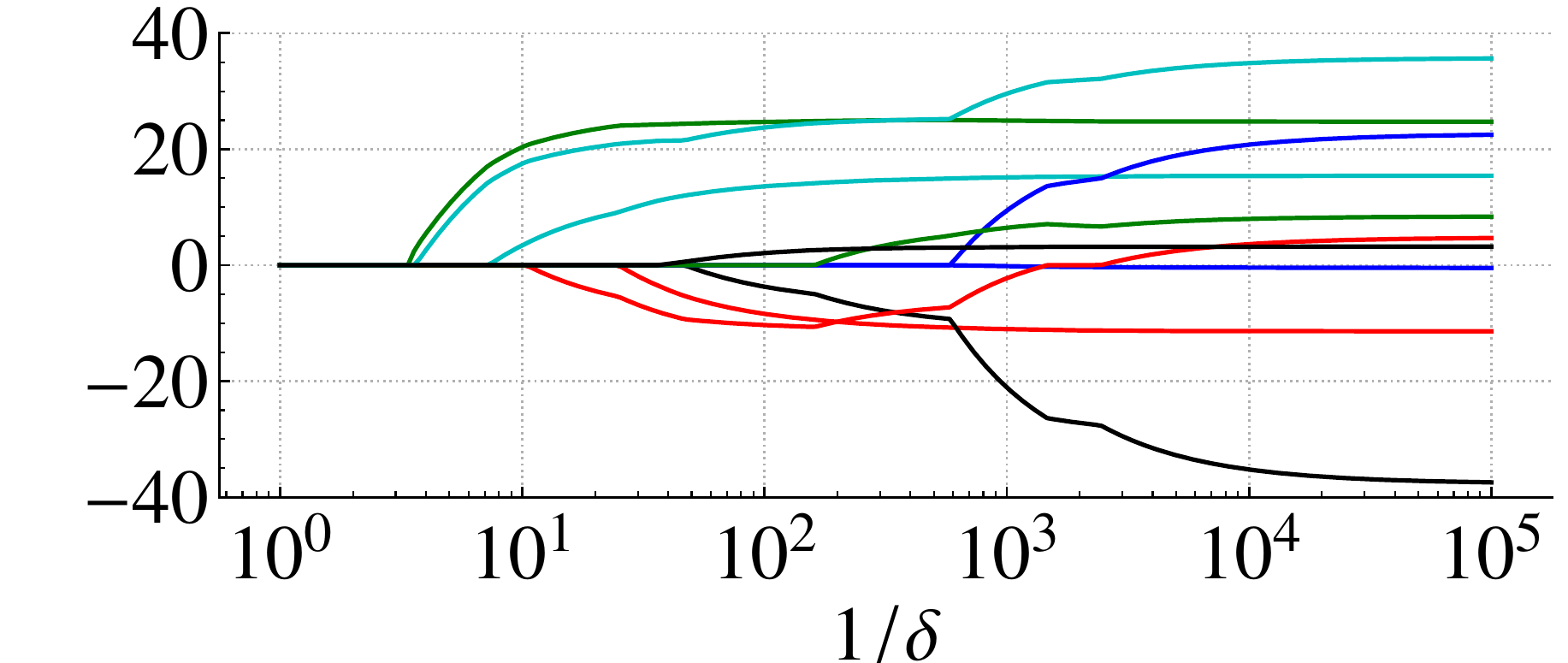}} \\
        \subfloat[ridge regression]{\includegraphics[width=0.42\textwidth]{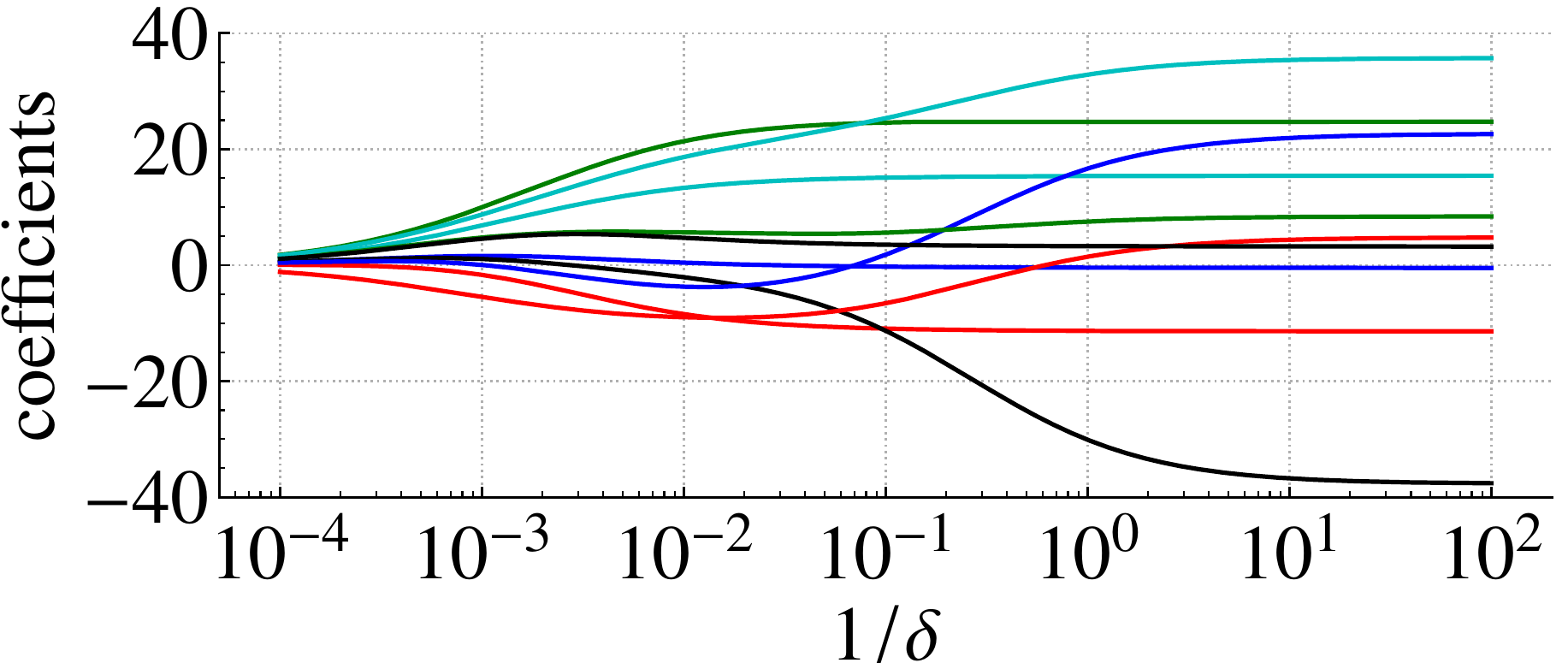}}
    \subfloat[$\ell_2$-adversarial training  ]{\includegraphics[width=0.42\textwidth]{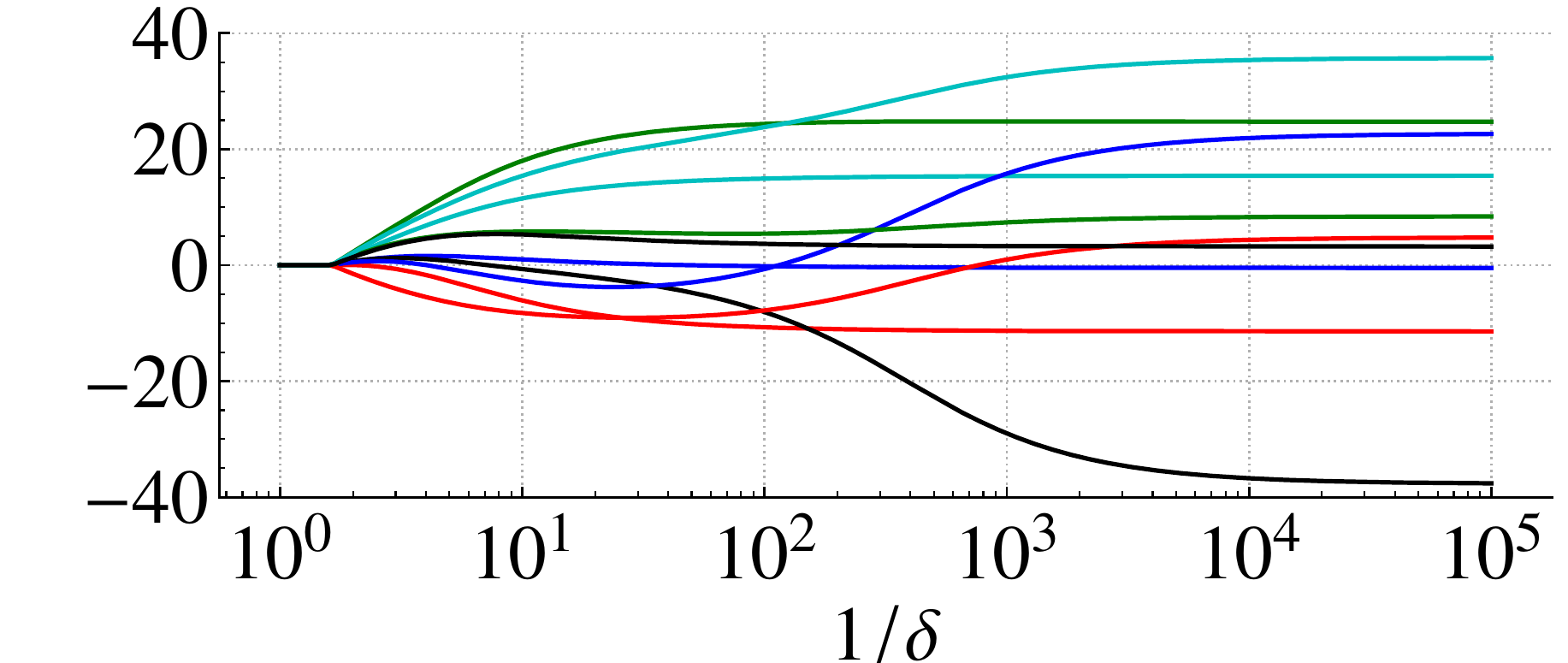}}
    \caption{\textbf{Regularization paths}. On the horizontal axis, we give the inverse of the regularization parameter (in log scale). On the vertical axis we show the coefficients of the learned linear model.}
    \label{fig:diabetes_model}
\end{figure*}

Restricting the analysis to regression allows for some important simplifications. Indeed, it was recently noted~\citep{ribeiro_overparameterized_2022,xing_generalization_2021,javanmard_precise_2020} that for linear \textit{regression}, adversarial training can be formulated as a quadratic  problem. We will refer to this as the \textit{dual formulation},
\begin{equation*}
\min_{\param}~\frac{1}{n}\sum_{i=1}^n\left(|y_i - \x_i^\trnsp\param| + \delta\|\param\|_1\right)^2.
\end{equation*}
The function minimized in the above problem is equivalent to~\eqref{eq:AdvTrainOrig} and  is easily shown to convex. This fact enables us to connect it to other well-known parameter-shrinking regularization methods. For instance, it hints at similarities between adversarial training constrained to the region $\{\dx: \|\dx\|_\infty\le \delta\}$ and lasso. Fig.\ref{fig:diabetes_model} (a)-(b) show experiments comparing the solutions and suggest that this similarity indeed exists.

The  dual formulation above is presented in~\citep{ribeiro_overparameterized_2022,xing_generalization_2021,javanmard_precise_2020} with the focus of analysing the generalization properties of adversarial attacks. This paper has a different focus: we are interested in the study of the solution to the adversarial training problem from an \emph{optimization perspective}. We show that significant insight can be gained into the problem from such a perspective.

\subsection{Setup}

More generally, we consider the case where the adversarial disturbance is bounded to a region that can be written as $\{\dx: \|\dx\|_p\le \delta\}$. Here $\|\cdot\|_p$ denotes the $\ell_p$ norm.\footnote{if $\vv{a} \in \R^n$ $p > 1$ then $\|\vv{a}\|_p = \left(\sum_{i=1}^n |\vv{a}_i|^p\right)^{\frac{1}{p}}$ and, for $p=\infty$, we define $\|\vv{a}\|_\infty = \max_{1\le i\le n}{|a_i|}$.}
Let us define the \textit{empirical adversarial risk} as
\begin{equation}
  \label{eq:empirical_advrisk}
    R_{p}^{\text{adv}}(\param; \delta) = \frac{1}{n}\sum_{i=1}^n{\max_{\|\dx_i\|_p \le \delta}(y_i - (\x_i + \dx_i)^\trnsp\param)^2}.
\end{equation}
The problem of finding the parameter $\param$ that minimizes the empirical adversarial risk $R_{p}^{\text{adv}}(\param; \delta)$ is referred to as the \emph{$\ell_p$-adversarial training} problem, which for linear regression can be simplified to the minimization of a sum of quadratic terms.
We will rely on this result, which is stated in the next lemma and was proved in~\citep{ribeiro_overparameterized_2022}. 
\begin{lemma}[Dual formulation adversarial risk]
Let $p \in [1, \infty]$ and $\frac{1}{p} + \frac{1}{q} = 1$, then 
\begin{equation}
\label{eq:advtraining-closeform}
R_{p}^{\text{adv}}(\param; \delta) =   \frac{1}{n}\sum_{i=1}^n\left(|y_i - \x_i^\trnsp\param| + \delta\|\param\|_q\right)^2
\end{equation}
\end{lemma}
An immediate consequence of this result is that for $p \in [1, \infty]$, $R_p^{\text{adv}}(\param; \delta)$ is convex in $\param$. 
The result  holds for any pair of complementary norms $(p, q)$. We note that this includes the values ${(p, q)= (1, \infty)}$ and ${(\infty, 1)}$. We also highlight that not only is the right-hand-side of~ \eqref{eq:empirical_advrisk} exactly the same as that of~\eqref{eq:advtraining-closeform}, but also the parameter $\delta$  is exactly the same, i.e., upper bound of the adversarial disturbance.

The proof provided in~\citep{ribeiro_overparameterized_2022} is based on Hölder's inequality ${|\param^\trnsp\dx| \le \|\param\|_p \|\dx\|_q}$ and on the fact that, given $\param$, we can construct $\dx$ such that the equality holds. The result also appear in other recent work in slightly less general settings:~\citet{xing_generalization_2021} presents it for a Gaussian distribution, and~\citep{javanmard_precise_2020} presents it for $\ell_2$ attacks.

\subsection{Summary of the contributions}
\label{sec:contributions}

We study the optimization problems that arise in adversarial training for linear regression. We will in Section~\ref{sec:adversarial-training} start by analysing the similarity between $\ell_\infty$-adversarial training and lasso. One interesting result we observed is that \emph{$\ell_\infty$-adversarial training produces sparse solutions}---see Fig.~\ref{fig:diabetes_model}(b). This fact is interesting but not clearly stated in the literature. Hence, our first contribution is a way to justify this observation.

\begin{enumerate}
\item[C-1]   We \emph{show that adversarial training is an instance of robust regression}.  We use this to connect it with square-root lasso, which  is also an instance of robust regression~\cite{xu_robust_2008}. Several examples are then used to illustrate the similarities between the methods.
\end{enumerate}
Note that the  dual formulation hold for any pair $(p, q)$ of complementary norms. Another possible pair would be $(p=2, q=2)$ and a similar argument can be used to point towards similarities between ridge regression and $\ell_2$-adversarial training. Such similarities can be observed in Fig.\ref{fig:diabetes_model} (c)-(d).

The second contribution of this work  is to study the behavior of adversarial training in overparameterized models, where there are more parameters than data ($m > n$). 

\begin{enumerate}
\item[C-2] We \emph{prove that  the minimum-norm interpolators are the solution to adversarial training for all values of $\delta$ below a threshold.}
\end{enumerate}

Here, we refer to the interpolators $\eparam^{\text{min-} \ell_1}$ and  $\eparam^{\text{min-} \ell_2}$, which are the solution that minimize, respectively, $\|\param\|_1$ and $\|\param\|_2$ and satisfy $y_i = \x_i ^ \trnsp \param $ for $i = 1, \dots n$.  
Both lasso and ridge with the regularization parameter approach the minimum $\ell_1$ and $\ell_2$-norm interpolators as the penalization on the norm goes to zero. We show that such connections also exist for adversarial training as expressed in the next theorem.
    
\begin{theorem}
\label{thm:min-norm-sols}
Assume that the matrix $\X \in \R^{n \times m}$ has full row rank (i.e, $\text{rank}(\X) = n$). There exists $\bar{\delta} > 0 $ such that: 

\begin{enumerate}
    \item  $\eparam^{\text{min-} \ell_1}$   minimizes $R_\infty^{\text{adv}}(\param, \delta)$, for all ${0 < \delta < \bar{\delta}}$;
    \item $\eparam^{\text{min-} \ell_2}$ minimizes $R_2^{\text{adv}}(\param, \delta)$, for all ${0 < \delta < \bar{\delta}}$.
\end{enumerate}
\end{theorem}

Interestingly, for adversarial training the interpolation does not occur only in the limit, but rather for all values $\delta$ smaller than a certain threshold. We prove the Theorem and discuss its consequences in  Section~\ref{sec:overparameterized-models}.

\subsection{Related work}

\paragraph{Adversarial training in linear models.} 

Gaining insight about adversarial attacks through the analysis of its behavior in linear models is now commonplace: for instance, ~\citet{tsipras_robustness_2019} and \citet{ilyas_adversarial_2019} use linear models to explain the conflict between robustness and high-performance models observed in neural networks; \citet{ribeiro_overparameterized_2022} uses these models to show how overparameterization affects robustness to perturbations; \citet{taheri_asymptotic_2021} derives asymptotics for adversarial training in binary classification.
\citet{javanmard_precise_2020} provides asymptotics for adversarial training in linear regression,
\citet{javanmard_precise_2020a} studies classification settings and 
\citet{hassani_curse_2022} for random feature regressions.
\citet{min_curious_2021} investigates how the dataset size affects adversarial performance.
\citet{yin_rademacher_2019} provides an analysis of $\ell_\infty$-attack on linear classifiers based on the Rademacher complexity.

\paragraph{Adversarial training and sparsity.}
In this paper, we establish that, in linear regression, $\ell_\infty$-adversarial attacks yield a sparse solution. We found this fact quite intriguing and were unable to find a clear reference for it, although several researchers in correspondence seem to intuitively expect it.  Related facts have since long been established for robust regression~\citep{xu_robust_2008} and in Section~\ref{eq:robust-regression} we show how adversarial training differs. Framing adversarial training as a robust optimization problem is also commonplace~\cite{shaham_understanding_2018, madry_deep_2018}. 
There are studies highlighting that sparse models can be more robust to adversarial attacks~\citep{guo_sparse_2018, gopalakrishnan_combating_2018}. \citet{xing_generalization_2021} propose to add an $\ell_1$-penalty to the adversarial loss to improve generalization, and~\citep{xing_adversarially_2021} proposes to add it as part of two-stage training with $\ell_1$ penalty being applied to the first stage and adversarial training in the second stage. \citet{xing_generalization_2021} points to the resemblance between the dual formulation of $\ell_\infty$-adversarial training and Lasso as a concluding remark but does not explore the connection in detail.

\section{Adversarial training and regularization}
\label{sec:adversarial-training}

In this section, we explore the relationship between adversarial training and parameter shrinking regularization methods. Indeed,  the dual form above yields a cost function that is remarkably similar to some of the traditional parameter shrinking regularization methods. For $\ell_{\infty}$ adversarial attacks (i.e. $q = 1$), the cost function $R^{\text{adv}}_{\infty}(\beta; \delta)$ in its dual form is similar to lasso regression~\citep{tibshirani_regression_1996}:
\[R^{\text{lasso}}(\param; \delta) = \frac{1}{n}\sum_{i=1}^n|y_i - \x_i^\trnsp\param|^2 +  \delta\|\param\|_1.\]
Both cost functions penalize the $\ell_1$ norm and the error between observation and predictions. We used $\delta$ to denote the lasso regularization parameter in order to highlight the similarity with the dual form.

We study the methods from two angles:
First, we analyze $\ell_\infty$-adversarial attacks and lasso through the lens of robust regression; Secondly, we compare the \textit{regularization paths} of the methods in an experimental setup, i.e., the path described by the solution as a function of the parameter $\delta$.

\subsection{Robust regression}
\label{eq:robust-regression}

In robust linear regression, the cost function being minimized during training is
\begin{equation}
    R^{\text{robust}}(\param; \mathcal{S}) = \max_{\mm{\Delta} \in \mathcal{S}}\|\y - (\X + \mm{\Delta})\param\|_2,
\end{equation}
where $\X\in\R^{n \times m}$ denotes the matrix for which the $i^{\text{th}}$ row is  $\x_i^\trnsp$ and $\y\in\R^{n}$ denotes the vector of outputs. 
Here, the disturbance matrix $\mm{\Delta}$ is constrained to belong to the  `uncertainty set' $\mathcal{S}$ . 

The robust regression framework allows us to establish interesting connections. On the one hand, robust regression is equivalent to square-root lasso for `feature-wise' uncertainty sets,
\begin{equation*}
\footnotesize
  \mathcal{C}_{2}(\delta) = \left\{
  \begin{bmatrix}
    \bigl| & & \bigl|\\
    \vv{\zeta}_1& \cdots& \vv{\zeta}_m \\
    \bigl| & & \bigl|
  \end{bmatrix}
  : \|\vv{\zeta}_i\|_2 \le \delta  \text{ for } i = 1, \cdots, m \right\}.
\end{equation*}

On the other hand, robust regression is equivalent to adversarial training for `sample-wise' uncertainty sets,
\begin{equation*}
\scriptsize
  \label{eq:uncertanty_adv}
  \mathcal{R}_{p}(\delta) = \left\{
  \begin{bmatrix}
    \llongdash & \dx_1 & \rlongdash\\
    &\vdots&\\
    \llongdash & \dx_n & \rlongdash\\
  \end{bmatrix}:
     \|\dx_i\|_p \le \delta  \text{ for } i = 1, \cdots, n \right\}.
\end{equation*}
This result is summarized in the next proposition.
\begin{proposition}\label{thm:robust_training_advtrain} 
  We have that for any $\delta$
 \begin{subequations}
 \begin{align}
    \text{\rm arg}\min_{\param}  R^{\text{robust}}(\param; \mathcal{C}_{2}(\delta))   &=  \text{\rm arg} \min_{\param} R^{\sqrt{\text{lasso}}}(\param, \delta); \\
    \text{\rm arg} \min_{\param}  R^{\text{robust}}(\param, \mathcal{R}_{p}(\delta)) & = \text{\rm arg}\min_{\param}  R_p^{\text{adv}}(\param, \delta).
    \end{align}
  \end{subequations}
  
\end{proposition}

Square-root lasso is the estimator minimizing,
  \begin{equation*}
   R^{\sqrt{\text{lasso}}}(\param, \delta) = \|\y - \X\param\|_2 + \delta \|\param\|_1.
  \end{equation*}
It optimizes a Lagrangian formulation of the same constrained problem as lasso. They have different properties however and the square-root lasso is pivotal and attain near-oracle performance without knowledge of the variance levels~\citep{belloni_square-root_2011}. Similarly, $\ell_\infty$-adversarial training has different high-dimensional properties when compared to lasso, see~\cite{xing_generalization_2021} for an analysis of some of these properties.

The first statement about robust regression is  proved in~\citet[theorem 1]{xu_robust_2008}. The second part follows by noting that 
\begin{equation*}
\begin{aligned}
R_p^{\text{adv}}(\param; \delta)
  &= \frac{1}{n}\max_{\mm{\Delta} \in  \mathcal{R}_{p}(\delta)}  \|\y - (\X + \mm{\Delta} )\param\|^2_2\\
  &= \frac{1}{n}\left(\max_{\mm{\Delta} \in  \mathcal{R}_{p}(\delta)} \|\y - (\X + \mm{\Delta} )\param\|_2\right) ^2.
\end{aligned}
\end{equation*}
Where the first equality follows from the definition of adversarial training and the last equality from the fact that the function $h(z) = \tfrac{1}{n} z^2$ is monotonically increasing for $z \ge 0$. Repeating the same argument, but now for the minimization, implies that $R_p^{\text{adv}}(\param; \delta)$ has the same minimizer as $\max_{\mm{\Delta} \in  \mathcal{R}_{p}} \|y - (\X + \mm{\Delta})\param\|_2$.

\subsection{Example: diabetes progression}

Regularization paths are plots that show the value of the coefficient estimates as the regularization parameter~$\delta$ is varied. Say, for lasso, we would define the solution
\begin{equation}
    \label{eq:ridge_solution}
    \eparam^{\text{lasso}}(\delta) =  \text{arg} \min_{\param} R^{\text{lasso}} (\param; \delta)
\end{equation}
i.e., the optimal parameter as a function of $\delta$. Such a function is visualized in a plot like the one in  Fig.~\ref{fig:diabetes_model}, with each one of the coefficients of $\eparam$ being displayed as a function of $\delta$. Let us now study the regularization paths for $\ell_2$ and $\ell_\infty$-adversarial training, ridge regression and lasso regression.

In Fig.~\ref{fig:diabetes_model} we show the so-called regularization paths for the above methods in the Diabetes dataset \citep{efron_least_2004}.  The dataset has $m=10$ baseline variables (age, sex, body mass index, average blood pressure, and six blood serum measurements), which were obtained for $n = 442$ diabetes patients. The model output is a quantitative measure of the disease progression one year after the baseline variables were measured. The dataset was introduced in the context of the study of lasso solutions.

A visual inspection of the regularization paths in panel (a) and (b) illustrate the similarities between  lasso and $\ell_\infty$-adversarial training. It also illustrates the fact that \emph{$\ell_\infty$-adversarial training yields sparse solutions}. 
Moreover, comparing panel (c) and (d) highlights the similarities between $\ell_2$ adversarial attacks (i.e. $q = 2$) and ridge  regression, which minimizes
\[R^{\text{ridge}}(\param; \delta) =  \frac{1}{n}\sum_{i=1}^n|y_i - \x_i^\trnsp\beta|^2 +  \delta\|\param\|_2^2.\]
One interesting difference is that while ridge regression coefficients slowly decay to zero as the amount of regularization is  increased, the transition is more abrupt in the case of $\ell_2$-adversarial training. At $\delta \approx 10^{-0.5}$, we see all coefficients collapsing to zero. 

We compute the regularization path of ridge regression and adversarial training by solving the optimization problem for a grid of $200$ values of $\delta$ equally spaced in the logarithmic space.  For lasso, we use the coordinate descent algorithm proposed by~\cite{friedman_regularization_2010}.  In this and all our numerical examples, the adversarial training solution is implemented by minimizing~\eqref{eq:advtraining-closeform} using CVXPY~\citep{diamond_cvxpy_2016}. 
The result is displayed for the different methods and, as $\delta \rightarrow \infty$, the coefficients should approach zero for all methods. On the other hand, as  $\delta \rightarrow 0$, the coefficients should approach the least-squares solution.

\section{Overparameterized models}
\label{sec:overparameterized-models}

When there are more parameters than data-point ($m > n$) there are multiple solutions to $\X \param = \y$.  Among these solutions, we highlight the solutions with the smallest $\ell_2$ and $\ell_1$ norm, respectively. That is, the minimum $\ell_2$-norm interpolator,
\[\eparam^{\text{min}-\ell_2} = \text{arg}\min_{\param}  \|\param\|_2 \mathrm{~~~subject~to~~~} \X \param = \y,\]
and the minimum $\ell_1$-norm interpolator,
\[\eparam^{\text{min}-\ell_1} = \text{arg}\min_{\param}  \|\param\|_1 \mathrm{~~~subject~to~~~} \X \param = \y.\]
These solutions are interesting for several reasons: The minimum $\ell_1$-norm interpolator is well studied in the context of `basis pursuit' and allows the recovery of low dimensional representations of sparse signals~\cite{chen_atomic_1998}. The interest in the minimum $\ell_2$ norm is more recent: such a solution is used in many recent papers where the double descent~\cite{belkin_reconciling_2019, hastie_surprises_2019}  and the benign overfitting phenomena~\cite{bartlett_benign_2020} are observed. Such studies shed light on the generalization of state-of-the-art models such as deep neural networks that (almost) perfectly fit the training data and still generalize well. Many of the empirical observations for state-of-the-art models can indeed be illustrated in simpler settings such as linear regression, see e.g.~\citet{hastie_surprises_2019,bartlett_benign_2020}. We point out that the minimum $\ell_1$-norm interpolator has also been studied in connection to the  `benign-overfitting' phenomena in~\cite{wang_tight_2022}.

Not only are such solutions interesting from an application point of view, there is also a connection between ridge regression and the minimum $\ell_2$-norm interpolator, namely the solution of ridge regression converges to the minimum norm solution as the parameter vanishes.
$\eparam^{\text{ridge}}(\delta) \rightarrow \eparam^{\text{min}-\ell_2}$ as 
$\delta \rightarrow 0^+$.
Similarly, there is a relation between the minimum $\ell_1$ norm solution and lasso.
The relation requires additional constraints because, for the overparameterized case, lasso does not necessarily has a unique solution. Nonetheless, it is proved in \cite[Lemma 7]{tibshirani_lasso_2013} that the lasso solution by the LARS algorithm satisfy $\eparam^{\text{lasso}}(\delta) \rightarrow \eparam^{\text{min}-\ell_1}$ as  $\delta \rightarrow 0^+$.

Interestingly, there is also a connection between adversarial training and such solutions. As we stated in Section~\ref{sec:contributions}, for a sufficiently small $\delta$, the solution of $\ell_2$ adversarial training equals the minimum $\ell_2$ norm solution; and, the solution of $\ell_\infty$-adversarial training equals the minimum $\ell_1$ norm solution. There is a notable difference though, while this happens only on the limit for ridge regression and lasso, for adversarial training this happens for all values $\delta$ smaller than a threshold, which we denote $\bar{\delta}$. Theorem~\ref{thm:min-norm-sols} (in the introduction) establishes this connection and follows from the propositions  below. 

 \begin{proposition}
 \label{sec:minimizer-lp-attacks}
For $0<\delta< \gamma_{\min{}}(\X)$, a  minima of $R_p^{\text{adv}}(\param; \delta)$ lies in the set $\{\param \in \R^m | \X \param = \y\}$, where  ${\gamma_{\min{}}(\bullet)}$ denotes the smallest (in magnitude) nonzero entry of a matrix, i.e., $\gamma_{\min{}}(\mm{X})  = \min_{i, j, |x_{i, j}|> 0} |x_{i, j}|.$
\end{proposition}

\begin{proof}[Proof of Theorem 2]
Full row rank guarantees that the set $\{\param \in \R^m | \X \param = \y\}$  is not empty.
Along this set, $R_p^{\text{adv}}(\param; \delta) = \delta^2\|\param\|_q^2$ by the dual formula~\eqref{eq:advtraining-closeform}.
Hence, it follows from the proposition that for $0< \delta < \gamma_{\min{}}(\X)$,
$$\min_{\beta} R_p^{\text{adv}}(\param; \delta)  = \min_{\X \param = \y} R_p^{\text{adv}}(\param; \delta) = \min_{\X \param = \y} \|\param\|_q 
\vspace{-5pt}$$
which is minimized by  $\eparam^{\min-{\ell_1}}$ and $\eparam^{\min-{\ell_2}}$ for, respectively, for $q=1$ and $q=2$.
\end{proof}

\subsection{Examples}
\label{sec:examples}

Before proving the proposition we give examples of its consequence. Unlike ridge regression and lasso that converges towards the interpolation solution, adversarial training goes through abrupt transitions and suddenly starts to interpolate the data.

\paragraph{Isotropic data model}
Figure~\ref{fig:distance-min-norm-sol} illustrates this phenomena  in synthetically generated data. In this case, we consider Gaussian noise and covariates: $\epsilon_i \sim \N(0, \sigma^2)$ and $\x_i \sim \N(0, r^2 \mm{I}_m)$  and the output is computed as a linear combination of the features contaminated with  additive noise: $y_i = \x_i^\trnsp \param+ \epsilon_i$.
We fix $m = 200$ and $n = 60$ and in Figure~\ref{fig:distance-min-norm-sol} we show what happens as we vary $\delta$ for models estimated using lasso, ridge and $\ell_\infty$ and $\ell_2$-adversarial training. In the Supplementary Material Section~\ref{sec:latent-feature} we consider a different synthetic set that uses features that are noisy observations of a lower-dimensional subspace, similar to~\citep[Section 5.4]{hastie_surprises_2019}.  The same abrupt  transition into interpolation can be observed.

\begin{figure}[ht]
    \centering
    \includegraphics[width=0.45\textwidth]{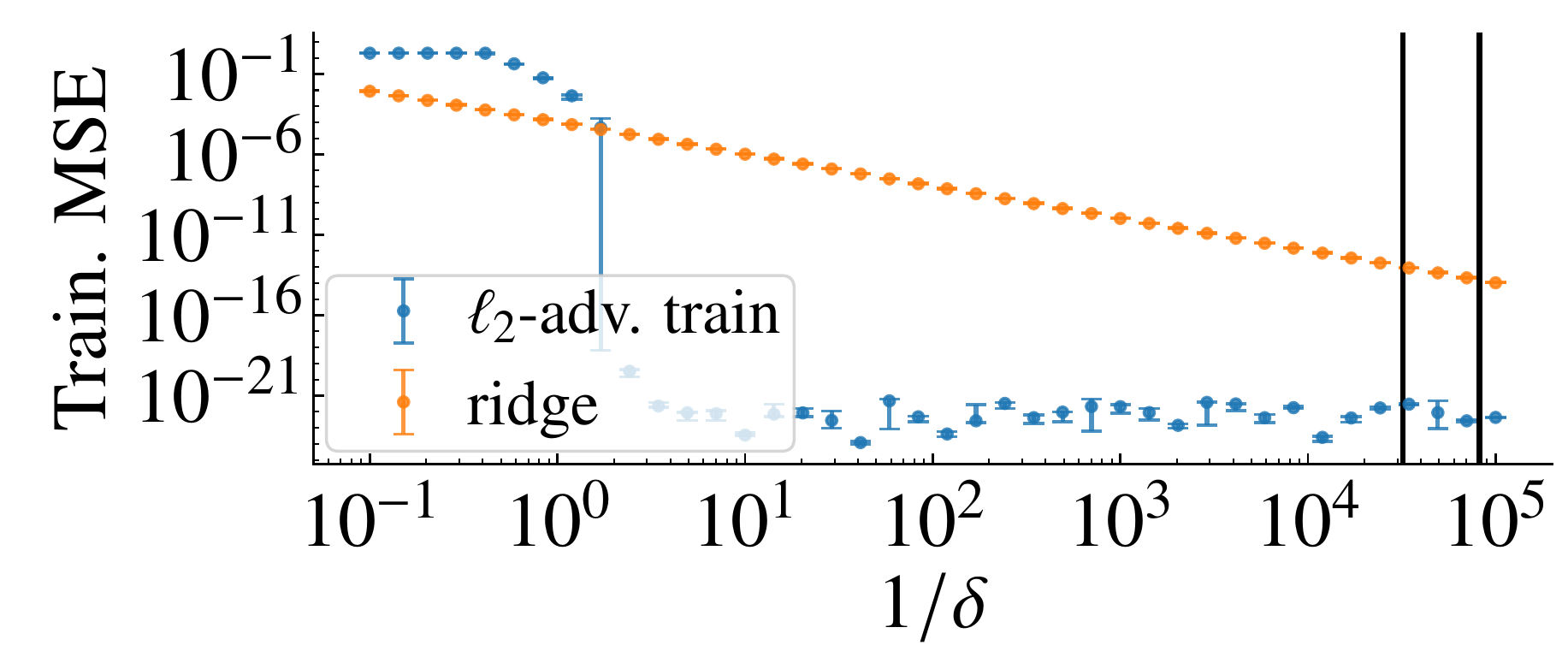}\\
    \includegraphics[width=0.45\textwidth]{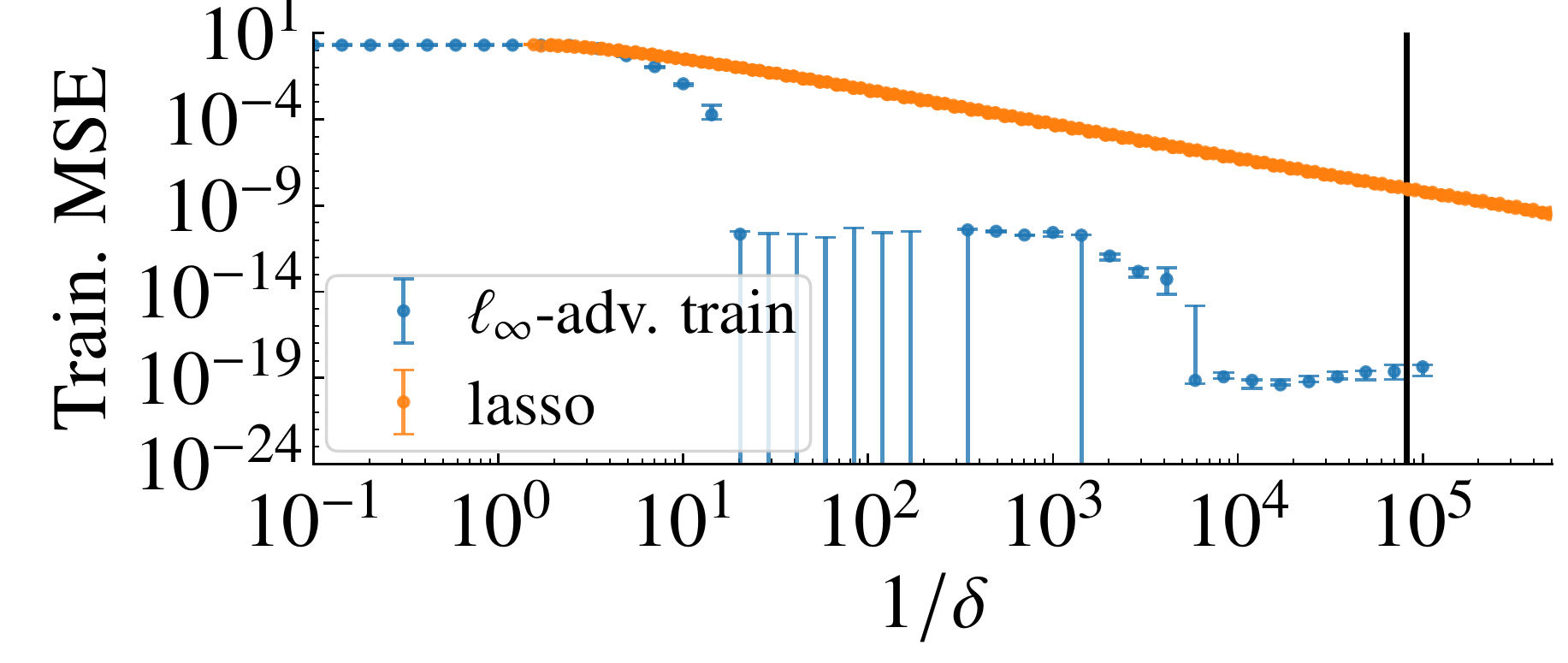}
    \caption{\textbf{Training MSE \textit{vs} regularization parameter}. \emph{Top:} for ridge and $\ell_2$-adversarial training. \emph{Bottom:}  for lasso and $\ell_\infty$-adversarial training The error bars give the median and the 0.25 and 0.75 quantiles obtained from numerical experiment (5 realizations). Vertical black lines give $\gamma_{\min{}} (\X)$ and, on the top figure,  $\gamma_{\min} (\mm{Q}\X)$ obtained from Propositions~\ref{sec:minimizer-lp-attacks} and~\ref{sec:minimizer-l2-attacks}.}
    \label{fig:distance-min-norm-sol}
\end{figure}

\paragraph{Phenotype prediction from genotype}

\begin{figure}[ht]
\begin{center}
  \includegraphics[width=0.45\textwidth]{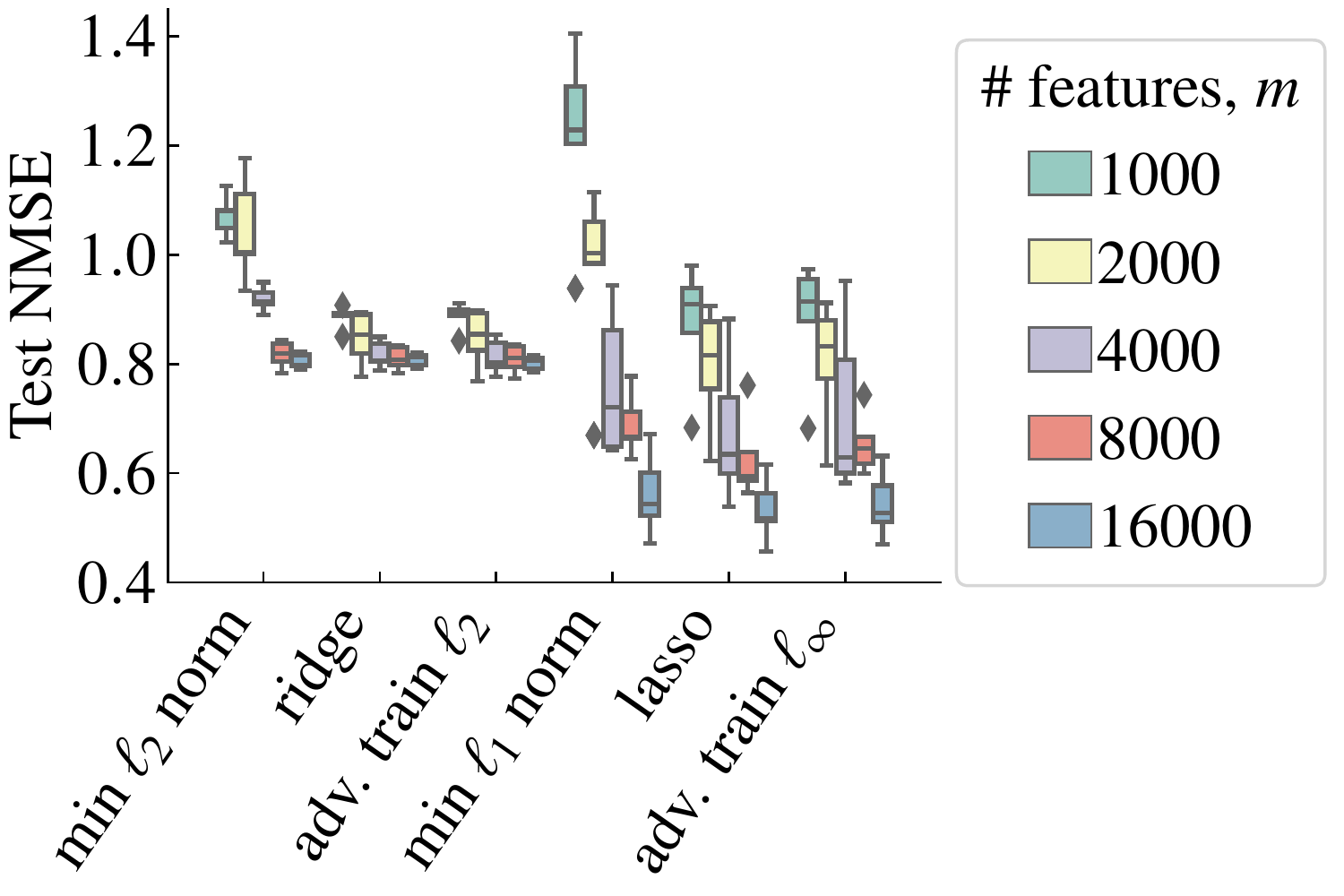}
    \includegraphics[width=0.45\textwidth]{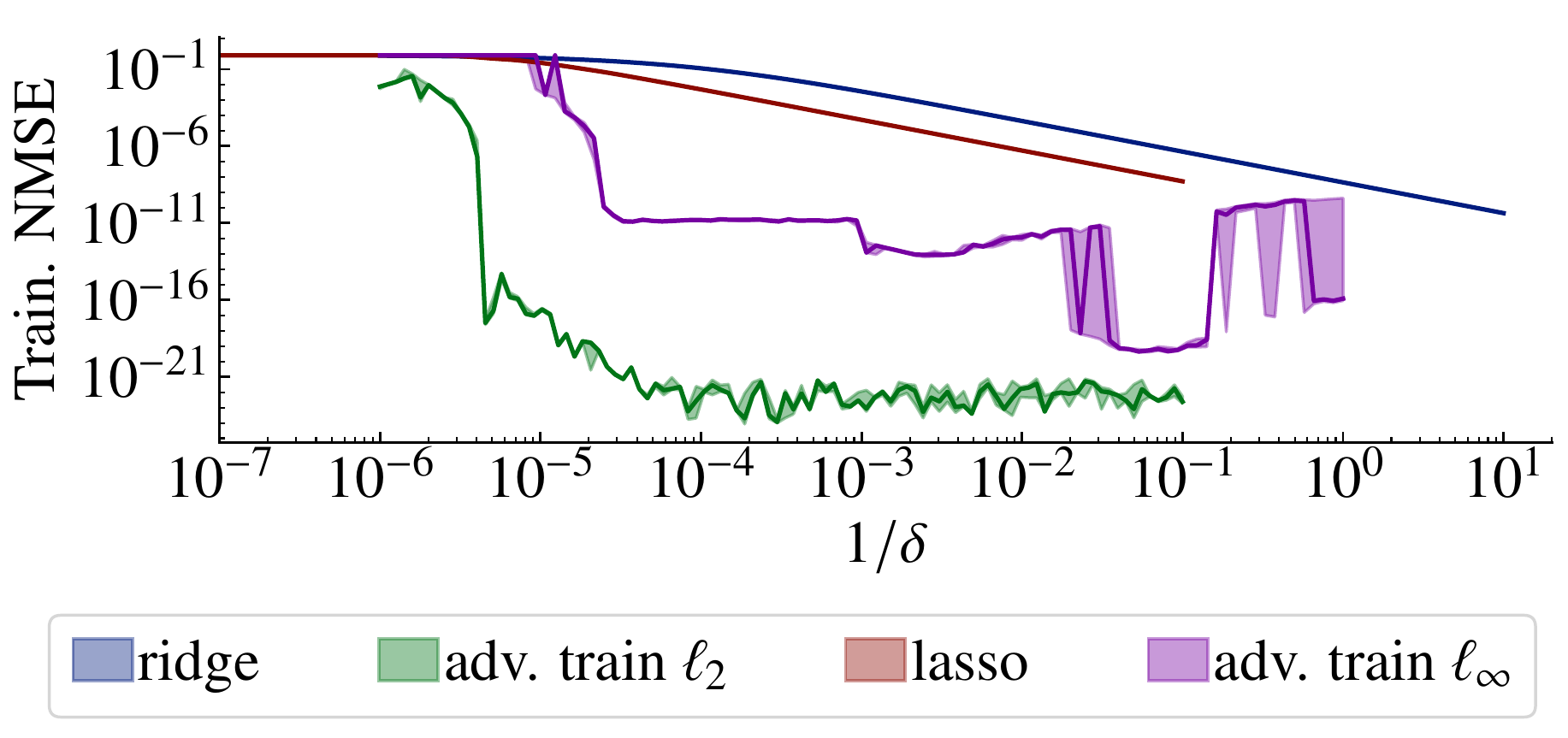}
    \end{center}
    \caption{\textbf{The normalized MSE (NMSE) on MAGIC dataset.} \emph{Top:} on the test set. With the best available regularization parameter $\delta$ obtained via grid search. \emph{Bottom:} on training as a function of $1/\delta$ for $m =16000$. We repeat the same experiment $5$ times with independently sampled features to obtain the median and inter-quartile range represented in the box plot on the top, and by the full line and colored region on the bottom.}
    \label{fig:magic}
  \end{figure}

We illustrate our method on the  Diverse MAGIC wheat dataset~\citep{scott_limited_2021} from the National Institute for Applied Botany. The dataset contains the whole genome sequence data and multiple phenotypes for a population of 504 wheat lines.  Here, we use a subset of the genotype to predict one of the continuous phenotypes. We use half of the individuals ($n=252$) for training and the remaining $252$ samples for testing. We have a binary input with values indicating whether each one of the 1.1 million nucleotides differs or not from the reference value. Closely located nucleotides tend to be correlated and we use only a pruned version of the genome as input to the model. We consider $m\in\{1000, 2000, 4000, 8000, 16000\}$ features, where $m$ is a (uniformly sampled) subset of the genome. 

Figure~\ref{fig:magic} (\emph{top}) gives the test error of each of the methods for different choices of $m$. For lasso, ridge and adversarial training, we use the best $\delta$ available for each method (obtained via grid search). We note that while for $m=1000$ optimally tuned lasso and $\ell_\infty$-adversarial training significantly outperform the corresponding minimum $\ell_1$-norm interpolator. As $m$ increases, the performance of the three different methods becomes quite similar. The same  apply to ridge, $\ell_2$-adversarial training and the minimum $\ell_2$-norm interpolator.

Interestingly, for $m = 1000$, the minimum $\ell_2$-norm has the best performance, but as $m$ increases, the minimum $\ell_1$-norm interpolator outperforms it. This is an interesting  natural example where the $\eparam^{\text{min}-\ell_1}$ can outperform the $\eparam^{\text{min}-\ell_2}$: the first seems to be better at incorporating more features without worsening the generalization capability of the model and it is an interesting counterpoint to~\citet{chatterji_foolish_2022}. Figure~\ref{fig:magic} (\emph{bottom}) shows training error as a function of the regularization parameter $\delta$ (for $m = 16000$).  While the training error decays slowly as we vary the amount of regularization $\delta$ for lasso  and ridge regression. For adversarial training,   we see again an abrupt drop in the training error. And after this abrupt drop the model starts to interpolate the training dataset (up to numerical precision).

\subsection{The dependency of $\bar \delta$ on $m$}

\begin{figure*}
    \centering
    \subfloat[ridge regression]{\includegraphics[width=0.4\textwidth]{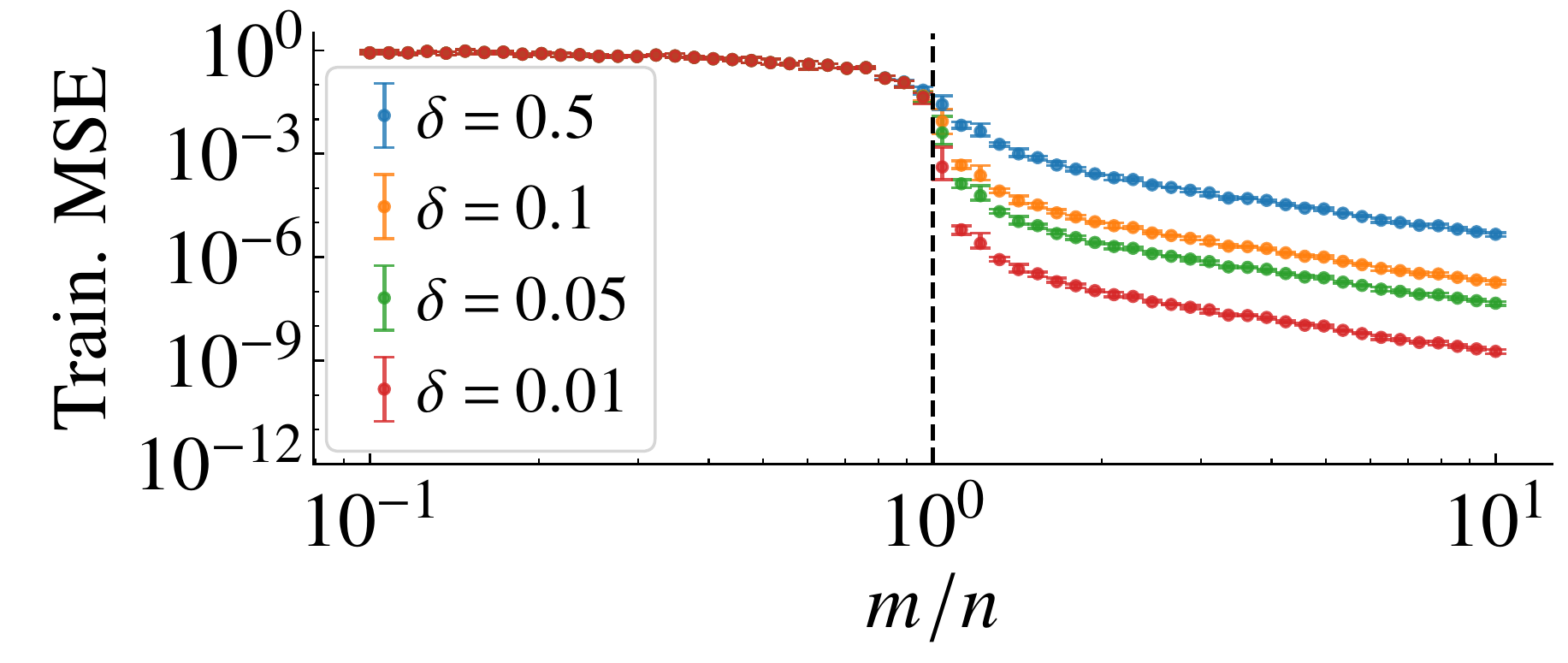}}
    \subfloat[$\ell_2$-adversarial training  ]{\includegraphics[width=0.4\textwidth]{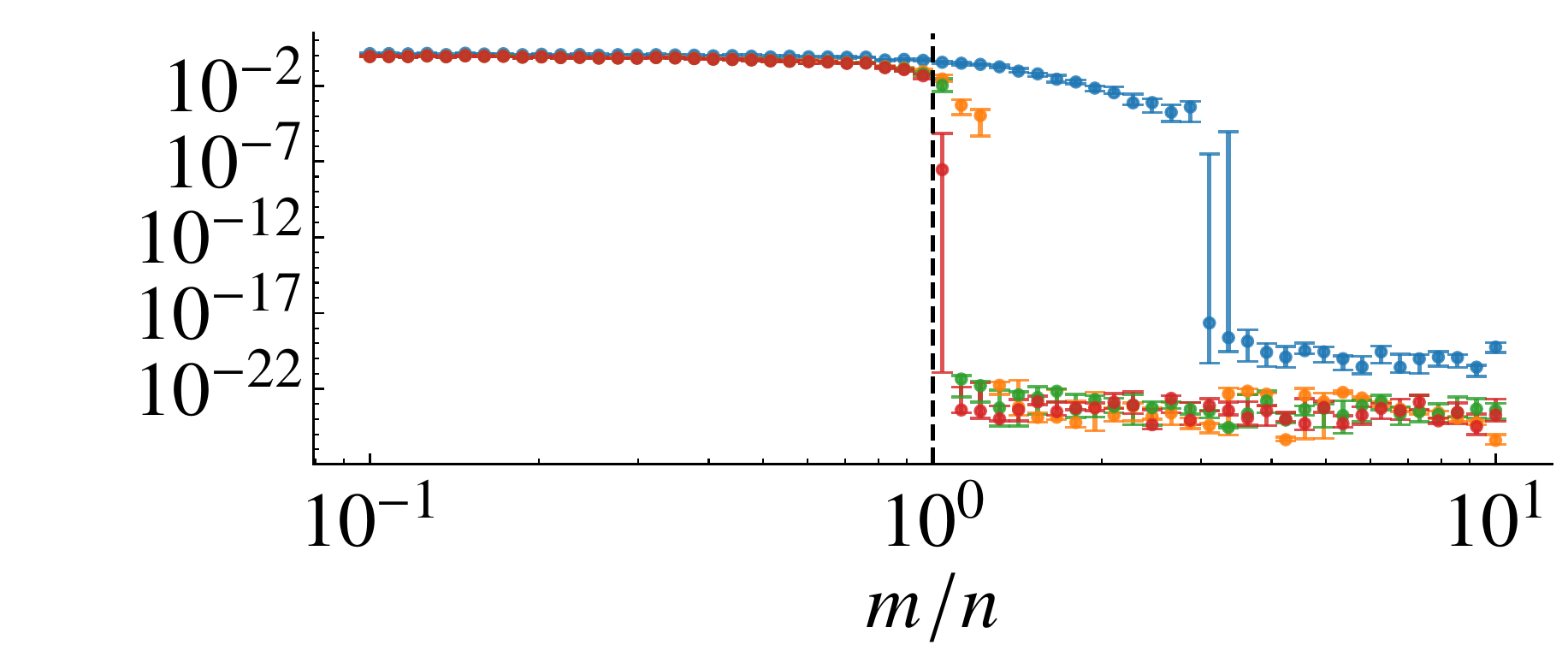}}\\
    \subfloat[lasso regression]{\includegraphics[width=0.4\textwidth]{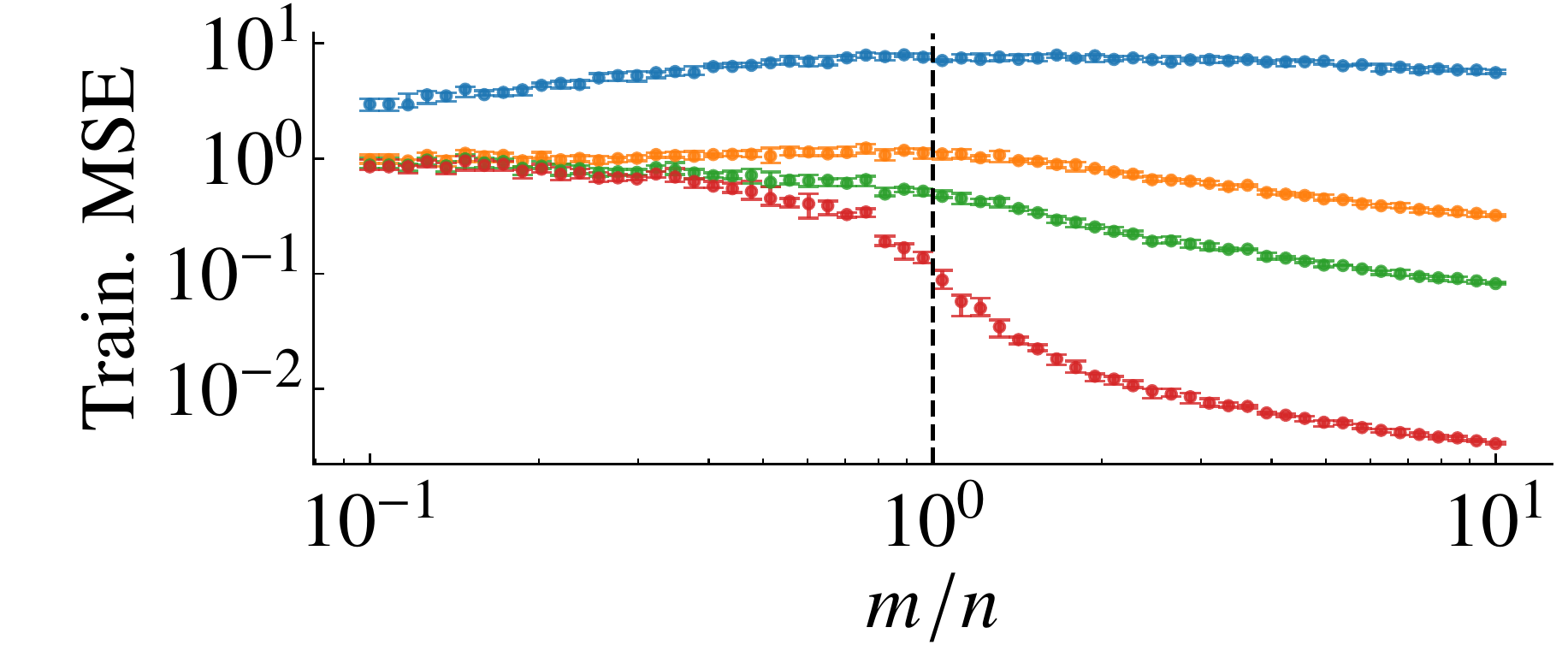}}
    \subfloat[$\ell_\infty$-adversarial training   ]{\includegraphics[width=0.4\textwidth]{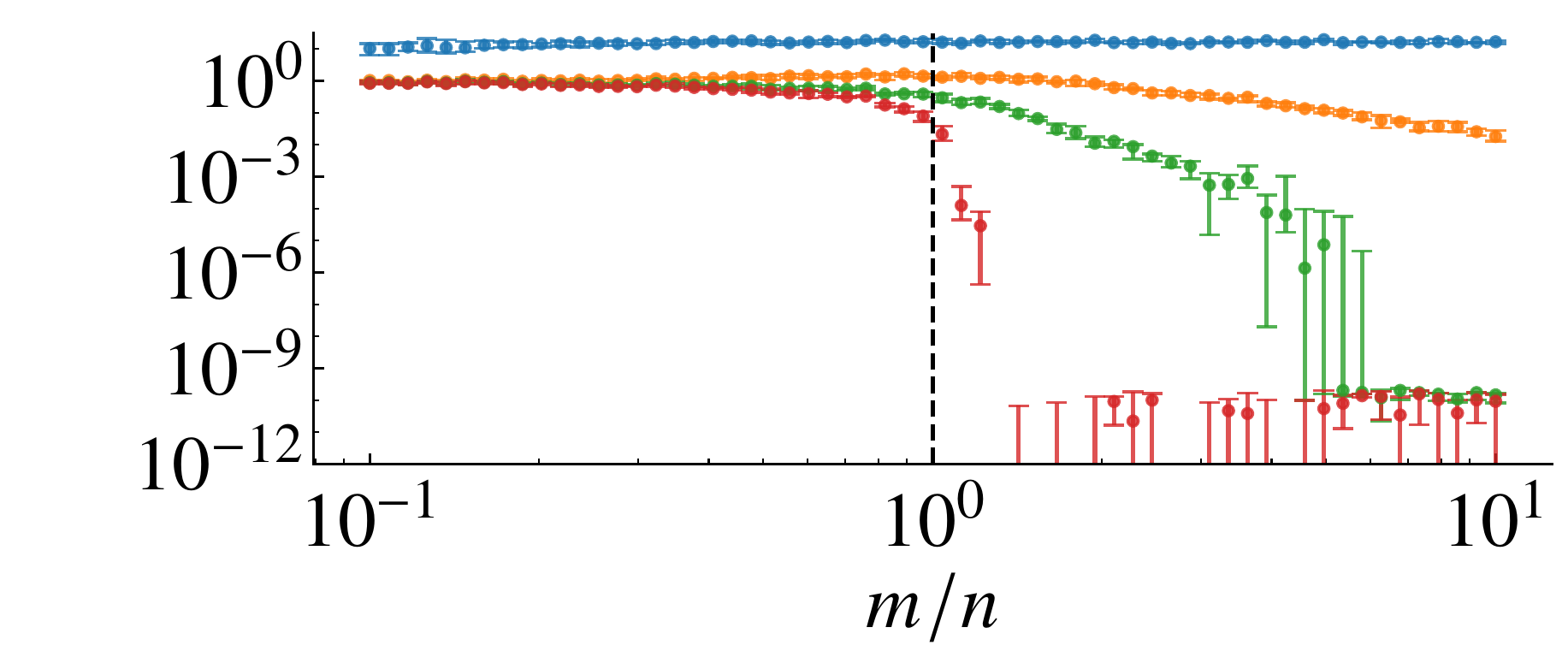}}\\
    \caption{\textbf{Mean square error in training data}. The error bars give the median and the 0.25 and 0.75 quantiles obtained from numerical experiment (10 realizations) and for regularization parameter $\delta = \{0.5,0.1,0.05,0.01\}$ with the colors corresponding to each case indicated in the legend. We use $r^2 = 4$, $\sigma^2 = 1$ and $n = 100$.}
    \label{fig:training-error-advtraining}
\end{figure*}

We observed in our experiment that there exists a threshold $\bar \delta$, such that for $0 < \delta < \bar \delta$  the solution to adversarial training is the minimum-norm solution that interpolates the data. 
In this sense, Theorem~\ref{thm:min-norm-sols} cast new light into minimum-norm interpolators: they are solutions minimizing the empirical adversarial risk for a disturbance $\bar \delta > 0$.  Hence, studying how $\bar \delta$ changes with $\X$ provides insight into the robustness of such solutions.

Note that Proposition~\ref{sec:minimizer-lp-attacks} only provides a sufficient condition and that there could be less conservative bounds on $\bar \delta$. Indeed, we observe empirically that often $\bar \delta > \gamma_{\min{}}(\X)$, e.g., the vertical black lines in  Figure~\ref{fig:distance-min-norm-sol}. The next proposition gives a tighter bound that holds for $\ell_2$-adversarial attacks.

\begin{proposition}
 \label{sec:minimizer-l2-attacks}
For $0<\delta<\gamma_{\min{}}(\X \mm{Q}^\trnsp)$, $\mm{Q}\in \R^{n \times m}$ any matrix with orthogonal rows that span the rows of $\mm{X}$, a  minima of $R_2^{\text{adv}}(\param; \delta)$ lies in the set $\{\param \in \R^m | \X \param = \y\}$.
\end{proposition}

Figure~\ref{fig:threshold from propositions} illustrates the bounds on the threshold obtained from both propositions for the same isotropic data. We obtain $\mm{Q}^\trnsp$ by the reduced $QR$ decomposition of $\mm{X}^\trnsp$ (as hinted by the notation). The bound $\gamma_{\min{}}(\X)$ decreases with $m$, while the bound $\gamma_{\min{}}(\X \mm{Q}^\trnsp)$ remains constant with it.  

However, we still believe that these bounds could be strengthened in future work.  The reason is shown in Figure~\ref{fig:training-error-advtraining}. There we keep $\delta$ constant and vary the number of features for the isotropic data model described in Section~\ref{sec:examples}. Adversarial training produces abrupt transitions in behavior by increasing $m$ with $\delta$ constant: giving us reason to believe $\bar \delta$ increases with $m$ in order to explain this observation.

\begin{figure}
    \centering
    \includegraphics[width=0.45\textwidth]{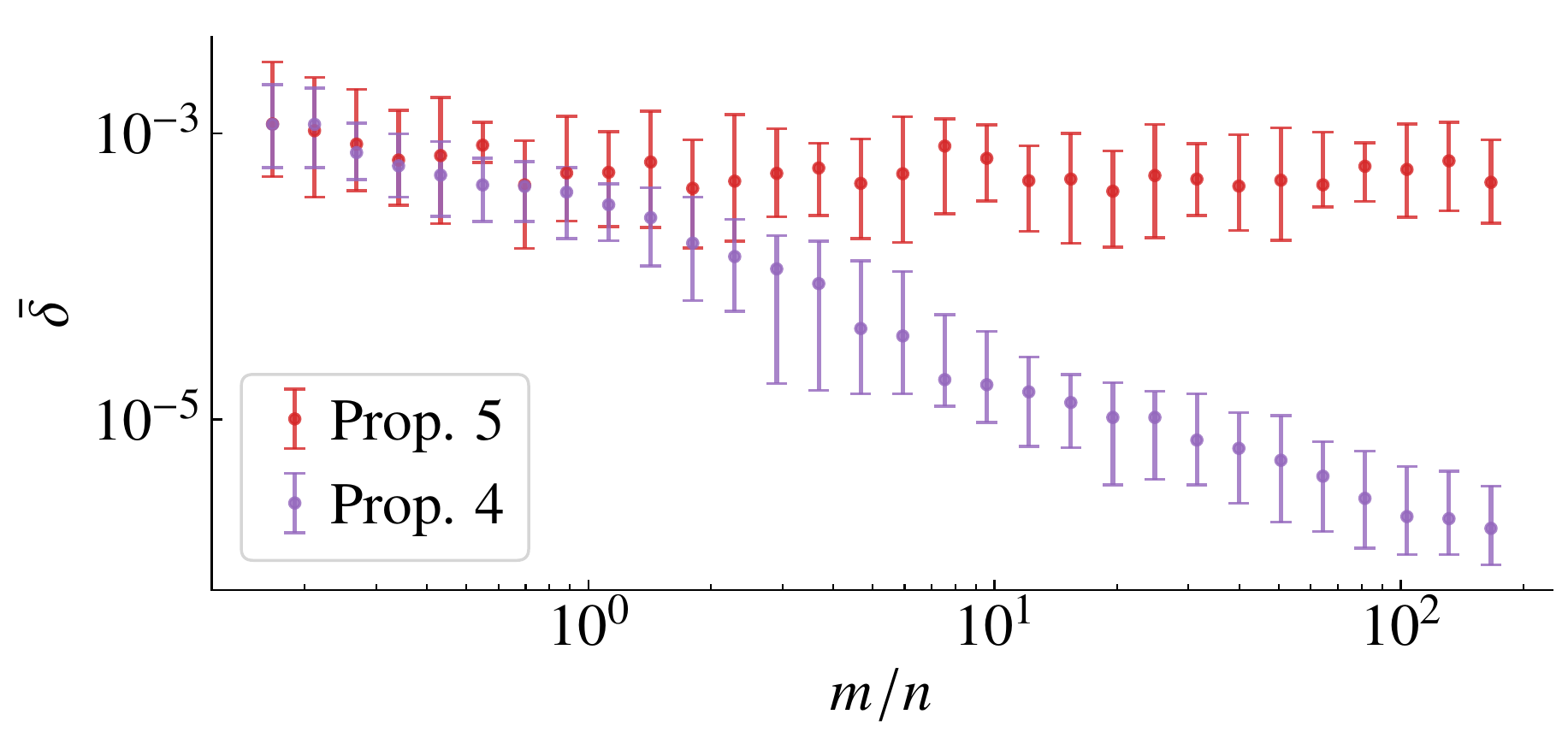}
    \caption{\textbf{Threshold \textit{vs} number of features.} Here the features $\X$ are generated as in the isotropic data model described in Section~\ref{sec:examples}.}
    \label{fig:threshold from propositions}
\end{figure}

\subsection{Proof of Proposition~\ref{sec:minimizer-l2-attacks}}

\begin{figure*}[t!]
    \centering
    \vspace{-10pt}
    \subfloat[$\param_\perp = \vv{0}$, variable $\delta$ ]{\includegraphics[width=0.45\textwidth]{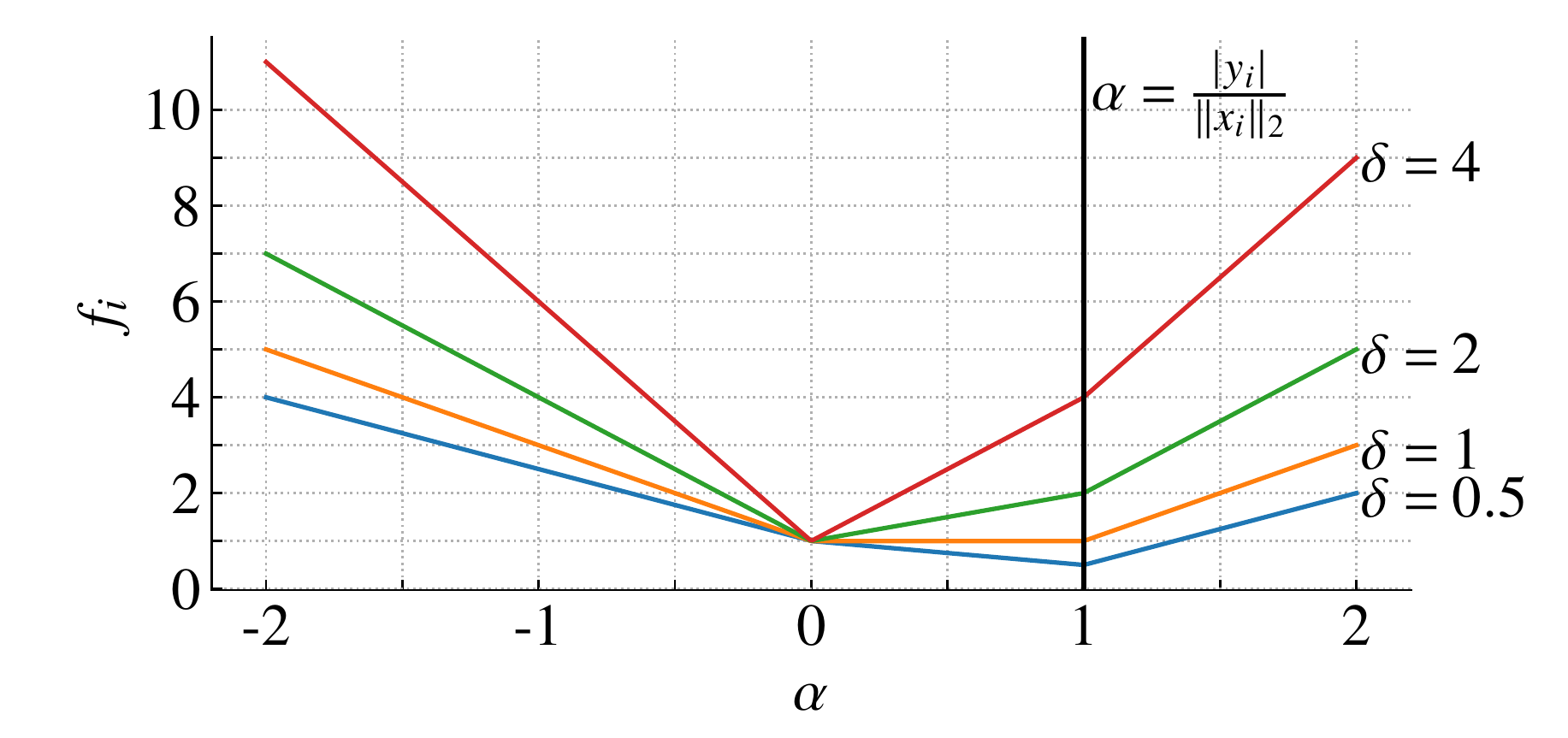}}
     \subfloat[$\delta =0.5$, variable $\param_\perp$]{\includegraphics[width=0.45\textwidth]{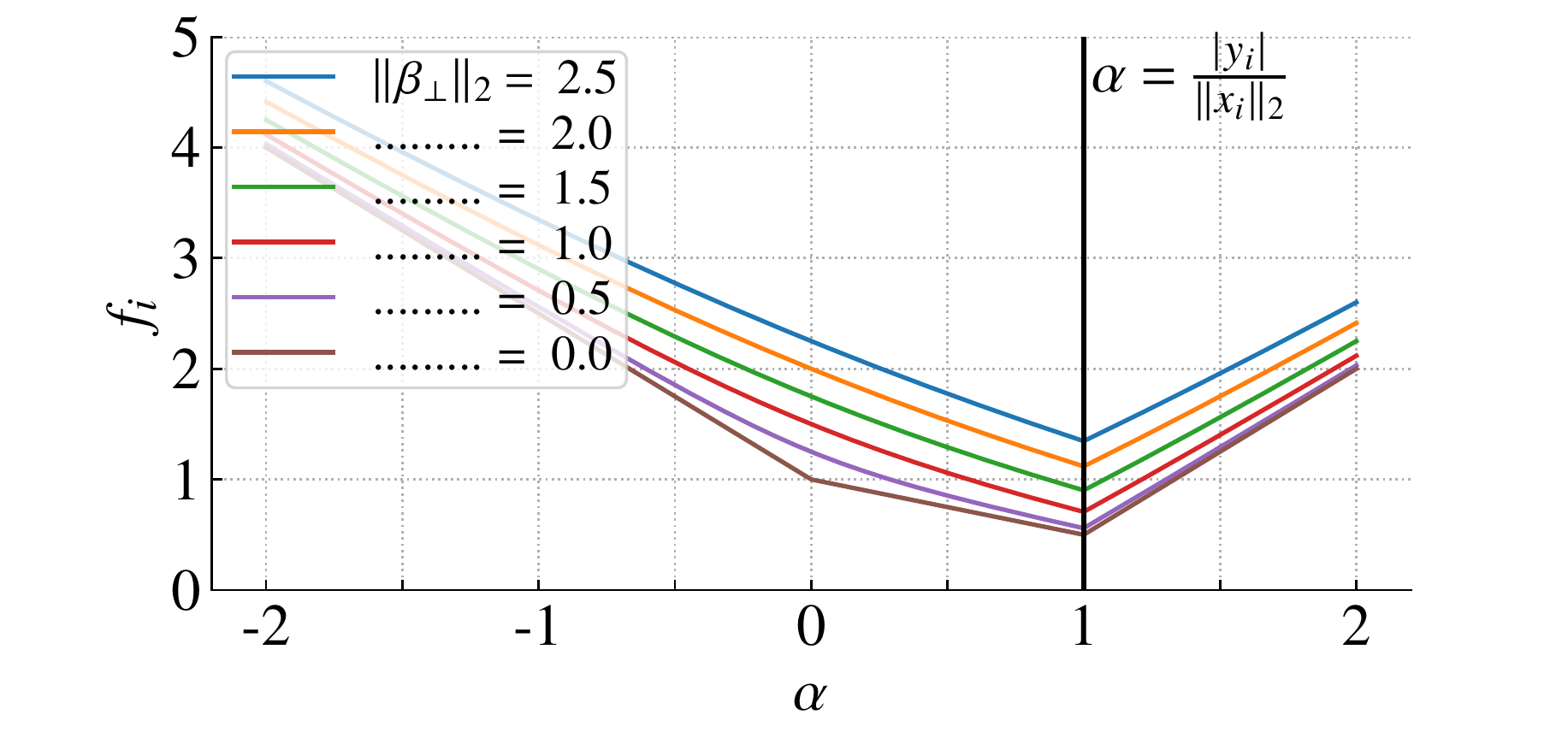}}
    \caption{\textbf{Function $f_i$.} The plot was generated for $|y_i| = \|x_i\|_2 = 1$. In (a), we consider $\param_\perp = \vv{0}$ and vary $\delta$. For $\delta < \|x_i\|_2$ the function has a minima at $\alpha = |y_i| / \|x_i\|_2$. For, $\delta > \|x_i\|_2$ it has a minima at $\alpha = 0$. For $\delta = \|x_i\|_2$ the function has a segment with slope 0 and is minimized at all the points in this segment. In (b), we consider $\delta = 0.5$ and vary $\|\param_\perp \|_2$. It has a minima at $\alpha = |y_i| / \|x_i\|_2$}
    \label{fig:fi_plot}
\end{figure*}

We use Equation~\eqref{eq:advtraining-closeform} as starting point for the proof. We will analyse the components that are being summed up. For convenience, let us denote the components:
$$f_{i}(\param) = |y_i - \x_i^\trnsp\param| + \delta\|\param\|_2.$$
such that $R_2^{\text{adv}}(\param; \delta) =   \frac{1}{n}\sum_{i=1}^n\left( f_{i}(\param)\right)^2$. We dropped $\delta$ and $p$ when referring to $f_{i}$ to make the notation more compact. We present first  the case where there is a $n=1$ (i.e., a single data point). This case gives insight into the proof and why abrupt transitions occur.

\begin{proof}[Proof of Proposition~\ref{sec:minimizer-l2-attacks} (n = 1)]
Given an arbitrary point $\x_1 \not = \vv{0}$,  $\param$ can be decomposed in two parts, one perpendicular to $\x_1$ and another one parallel to it. For instance, if we let $\vv{q}_1 = \text{sign}{(y_1)} \frac{\x_1}{\|\x_1\|_2 }$,  then we can write $\param = \alpha_1 \vv{q}_1 + \param_\perp$, where we denoted the perpendicular component as $\param_\perp$. Since the two components are perpendicular we have that $\|\param\|_2 > \alpha_1$ as long as $\param_\perp \not=\vv{0}$ and $f_1(\alpha_1 \vv{q}_1) \le f_{1}(\param)$. Hence the function is necessarily minimized along the subspace spanned by the vector $\vv{q}_i$. Now, along this line, the function $f_{1}$ is piecewise linear with three segments.  We show the function in Fig.~\ref{fig:fi_plot}(a) for different values of $\delta$. For $\delta < |\x_1^\trnsp \vv{q}_1| = \|\x_1\|_2 $, we have a minimum at $\alpha_1 = |y_1| / \|\x_1\|_2 $. Hence, $\eparam = \aarg \min_{\param}f_{1}(\param)  = \alpha_1 \vv{q}_1 = y_1 \frac{\x_1}{\|\x_1\|_2^2}$ and $\x_1^\trnsp \eparam = y_1$. And this conclude the proof for $n = 1$.
\end{proof}

It is easy to understand why abrupt transition occur by looking at the function displayed in Fig.~\ref{fig:fi_plot}(a). One of the three segments changes the slope sign as we change $\delta$ and the minimum of the function changes abruptly. Considering that $\param_\perp \not= 0$ is also instructive and gives the intuition why the result holds for $n\ge 1$.  If $\param_\perp \not= 0$, then
$$f_{1}(\alpha_1\vv{q}_1 + \param_\perp)  = |y_1 - \alpha_1  \x_1^\trnsp \vv{q}_1| +  \delta \sqrt{\alpha_1^2 + \|\param_\perp\|^2_2}.$$
Figure~\ref{fig:fi_plot}(b) display the value of $f_{i}$ as a function of $\alpha_1$ for different choices of $\|\param_\perp\|_2$.  The minimum will still occur for $ \x_1^\trnsp \vv{q}_1 \alpha_1 = y_1 $, as long as $\delta < | \x_1^\trnsp \vv{q}_1|$. The idea why the result holds for $n>1$ is that, while we might have to adjust for more than one coefficient $\alpha$, the coefficients can still be adjusted `independently', because adjusting the other coefficients will not really alter the minimum of the function $f_{1}$.

This intuition can be formalized using subderivatives. The subderivatives of a function $h: \R^m \rightarrow \R$ is the set $$\partial h(\x_0) = \{\vv{v} \in \R^m | h(\x) - h(\x_0) \ge \vv{v}( \x - \x_0)  \forall \x \in \R^m  \}.$$ 
The properties of subderivatives used here can be found on standard convex analysis textbooks, see for instance~\citet[chapter 4]{bertsekas_convex_2003} or \citep{clarke_optimization_1990, boyd_subgradients_2022}.

\begin{proof}[Proof of Proposition~\ref{sec:minimizer-l2-attacks} (general case)]

Let $\{\vv{q}_i\}_{i=1}^n$ be the rows of $\mm{Q}$, i.e., a set of unitary mutually orthogonal vectors that span the set of vectors $\{\x_i\}_{i=1}^n$. If we denote $\vv{r}_{i} = \x_i^\trnsp \mm{Q}$, any vector $\x_i$ can be written as $\x_i = \mm{Q}\vv{r}_{i}$.Moreover, we can write $\param =  \mm{Q}^\trnsp \vv{\alpha} + \param_\perp$ where $\param_\perp$ is perpendicular to all $\x_i$. The same argument as before yields that the minimizer of $f_{i}$ always occurs for $\param_\perp = \vv{0}$. Hence, without loss of generality we can assume  $\param = \mm{Q}^\trnsp \vv{\alpha}$. Hence
$$f_{i}\left(\mm{Q}^\trnsp \vv{\alpha}\right) =\left|y_i -  \vv{r}_i^\trnsp \vv{\alpha}\right| + \delta\|\vv{\alpha}\|_2.$$
Now, the subderivative of the above is given by following set $\partial f_{i}(\mm{Q}^\trnsp \vv{\alpha}) \subset \R^n$, 
$$\partial f_{i}(\mm{Q}^\trnsp \vv{\alpha}) = \vv{r}_i \partial |y_i - \vv{r}_i^\trnsp \vv{\alpha}|  +\delta \partial \|\vv{\alpha}\|_2$$
where $\partial \|\vv{\alpha}\|_2 = \left\{\frac{\vv{\alpha}}{\|\vv{\alpha}\|_2}\right\}$ if $\vv{\alpha} \not= 0$ and $\{\vv{v}: \|\vv{v}\|_2 \le 1\}$ otherwise. And
$$\partial |y_i - \vv{r}_i^\trnsp \vv{\alpha}| = \begin{cases}
\{1\}\text{ if } y_i > \vv{r}_i^\trnsp \vv{\alpha}\\\
\{-1\}\text{ if } y_i < \vv{r}_i^\trnsp \vv{\alpha}\\
\{\gamma: \gamma \in [-1, 1]\}
\text{ if }  y_i = \vv{r}_i^\trnsp \vv{\alpha}
\end{cases}
$$
Let $r_{i, j}$ be the $j^{\text{th}}$ component of $\vv{r}_j$.
Notice that, if $\delta \le |r_{i, j}|$ for all $j$ for which $|r_{i, j}| > 0$, then $\vv{0} \in \partial f_{i}(\mm{Q}^\trnsp \vv{\alpha})$ if and only if $y_i = \vv{r}_i^\trnsp \vv{\alpha}$.  Notice that any $\delta$ satisfying the hypothesis of the theorem satisfies such inequality.
Now, using the chain rule and additivity of subderivatives:
\begin{equation}
   \partial  R_2^{\text{adv}}  (\mm{Q}^\trnsp \vv{\alpha}) =   \frac{2}{n}\sum_{i=1}^n f_{i} (\mm{Q}^\trnsp \vv{\alpha}) \partial f_{i} (\mm{Q}^\trnsp \vv{\alpha}).
\end{equation}
Hence for $\delta$ satisfy the hypothesis of the theorem and $y_i = \vv{r}_i^\trnsp \vv{\alpha}$ than $\vv{0} \in \partial f_{i}(\mm{Q}^\trnsp \vv{\alpha})$. If this hold for all $i$, then  $\vv{0} \in \partial  R_2^{\text{adv}}(\mm{Q}^\trnsp\vv{\alpha})$. By the optimality condition for subderivatives, a point where this is satisfied local minima, and, since the function is convex, a global minimum. Hence, any point in the set $\{\mm{Q}^\trnsp\vv{\alpha} | \mm{r}_i \vv{\alpha} = y_i, \forall i\}$ will be a minima. Now, $\vv{r}_i  = \mm{Q} \x_i$ and this set is the same as the set $ \{\param | \X \vv{\beta} = \y\}$. The theorem follows.
\end{proof}

\subsection{Proof of Proposition~\ref{sec:minimizer-lp-attacks}}

The proof follows a similar outline. Let us denote:
$$g_{i}\left(\param\right) =\left|y_i -  \x_i^\trnsp \param\right| + \delta\|\param\|_q.$$
such that $R_p^{\text{adv}}(\param; \delta) =   \frac{1}{n}\sum_{i=1}^n\left( g_{i}(\param)\right)^2$. Now, the subderivative of  $g_i$ is given by following set $\partial g_{i}(\param) \subset \R^m$, 
$$\partial g_{i}(\param) = \vv{x}_i \partial |y_i - \x_i^\trnsp \param|  +\delta \partial \|\param\|_q$$
where 
$$\partial \|\vv{\beta}\|_q = \left\{\vv{g}~|~\|\vv{g}\|_p\le 1, \vv{g}^\trnsp \vv{\beta} \right\}
$$
Note that, by the above definition for any $\vv{g} \in \partial \|\vv{\beta}\|_q$, ${g_i \le 1}$. Hence, if $\delta \le |x_{i,j}|$ for every $j$, then $\vv{0} \in \partial g_{i}(\param)$ if and only if $y_i = \x_i^\trnsp \param$. And the rest of the argument follows exactly the same

\section{Conclusion}
\label{sec:conclusion}

In this work, we have studied adversarial training for linear regression from an optimization perspective. By viewing the method as an instance of robust regression we have established a connection to parameter-shrinkage methods and minimum-norm interpolators. These properties in linear regression may also be relevant in many nonlinear cases. At the moment, neural networks are still where most research on adversarial training is conducted.  Nonetheless, we believe understanding adversarial training in linear settings is an important milestone. Moreover, adversarial training in linear regression is an interesting method in its own right. Linear models are still among the most commonly used techniques and the method is a practical solution that can be of interest.  The method has a clear interpretation and addresses worst-case scenarios.  Still, a thorough theoretical characterization of the method and more technical  developments are needed to make it competitive.
For instance, while convex, the optimization is still quite inefficient. Coordinate descent~\cite{friedman_pathwise_2007, friedman_regularization_2010} and LARS~\cite{efron_least_2004}  are clever solutions that allow the entire regularization path to be efficiently estimated for lasso and elastic net. Having tailored solvers for linear adversarial training could make the method more computationally attractive and useful.

\acknowledgments{
The authors would like to thank  Carl Nettelblad for very fruitful discussions throughout the work with this research. This research was financially supported by the \emph{Kjell och M{\"a}rta Beijer Foundation} and by the Swedish Research Council through the projects: \emph{Deep probabilistic regression -- new models and learning algorithms }(contract number: 2021-04301)  and  \emph{Counterfactual Prediction Methods for Heterogeneous Populations} (contract number: 2018-05040). }

\printbibliography


\newpage

\onecolumn

\appendix

\aistatstitle{Surprises in adversarially-trained linear regression: \\ \textit{Supplementary Material}}
\thispagestyle{empty}


\setcounter{equation}{0}
\renewcommand{\theequation}{S.\arabic{equation}}%

\setcounter{figure}{0}
\renewcommand{\thefigure}{S.\arabic{figure}}%

\section{Latent feature models}
\label{sec:latent-feature}

We include here results for a different synthetic dataset: the ``latent space'' feature model described in~\citet[Section 5.4]{hastie_surprises_2019}. The features~$x$ are noisy observations of a lower-dimensional subspace of dimension~$d$. A vector in this \textit{latent space} is represented by $z \in \R^d$. This vector is indirectly observed via the features $x \in \R^m$ according to
\begin{equation}
    \label{eq:latent-model-features}
    x = W z + u,
\end{equation}
where $W$ is an $m \times d$ matrix, for $m \le d$. We assume that the responses are described by a linear model in this latent space
\begin{equation}
    \label{eq:latent-model-outputs}
    y = \theta^\trnsp z + \xi,
\end{equation}
where $\xi \in \R$ and $u\in \R^m$ are mutually independent noise variables. Moreover, $\xi \sim \N(0, \sigma_{\xi}^2)$ and $u \sim \N\left(0, I_m\right)$. As in~\citet{hastie_surprises_2019}, we consider the features in the latent space to be isotropic and normal $z_i = \N\left(0, I_d\right)$ and choose $W$ such that its columns are orthogonal, $W^\trnsp W = \frac{m}{d} I_d$, where the factor $ \frac{m}{d}$ is introduced to guarantee that the signal-to-noise ratio of the feature vector $x$ (i.e. $\frac{\|W z_i\|_2^2}{\|u_i\|_2^2}$) is kept constant.

Here we show experiments for $\sigma_{\xi}=0.1$ and $d = 20$. All the experiments are analogous to the experiments performed for the isotropic case. Fig.~\ref{fig:distance-min-norm-sol-latent} shows the training mean square error (MSE) for $m = 200$ and $n = 60$ fixed as we vary $\delta$. It also shows the thresholds for interpolation obtained from Propositions~\ref{sec:minimizer-lp-attacks} and~\ref{sec:minimizer-l2-attacks}. Fig.~\ref{fig:training-error-advtraining-latent} shows the training mean square error (MSE) of the learned predictors as we keep $\delta$ and  $n$ constant and vary the number of features $m$.

\begin{figure}[ht]
    \centering
    \includegraphics[width=0.45\textwidth]{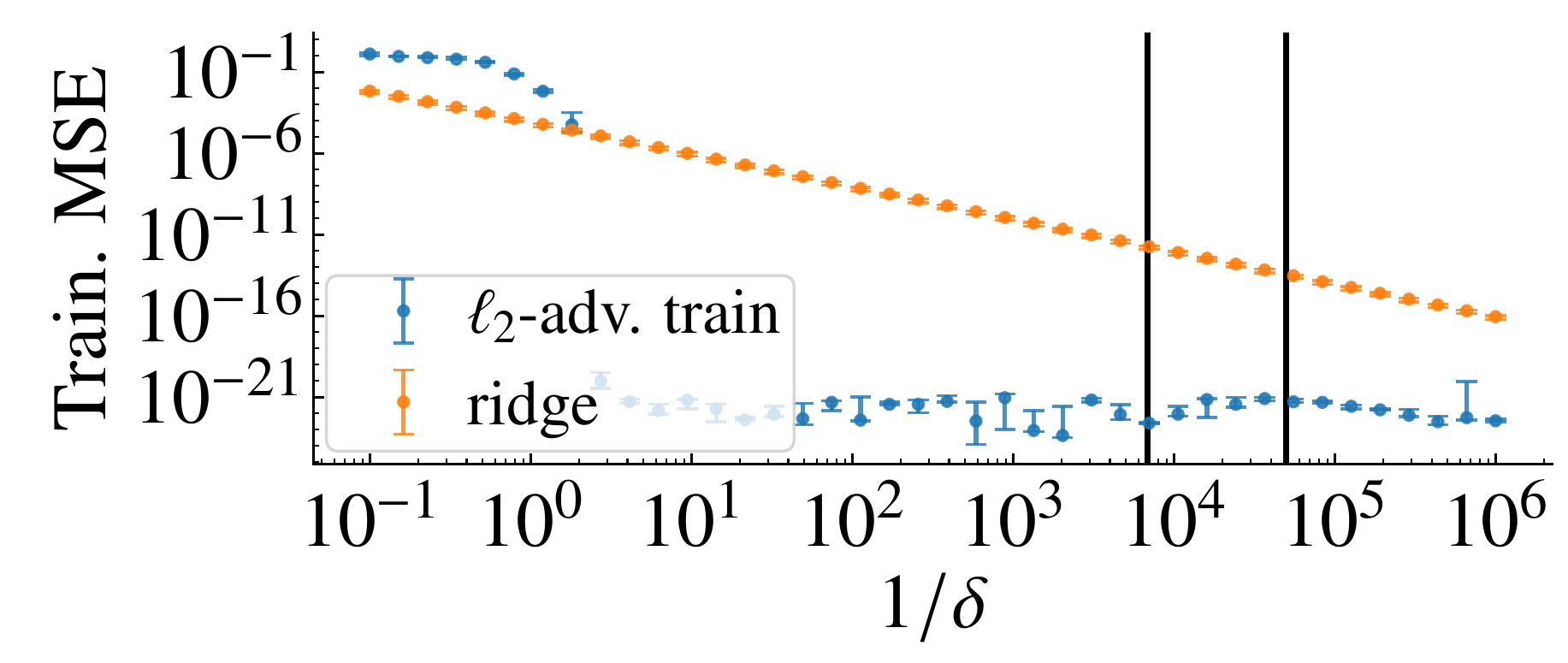}
    \includegraphics[width=0.45\textwidth]{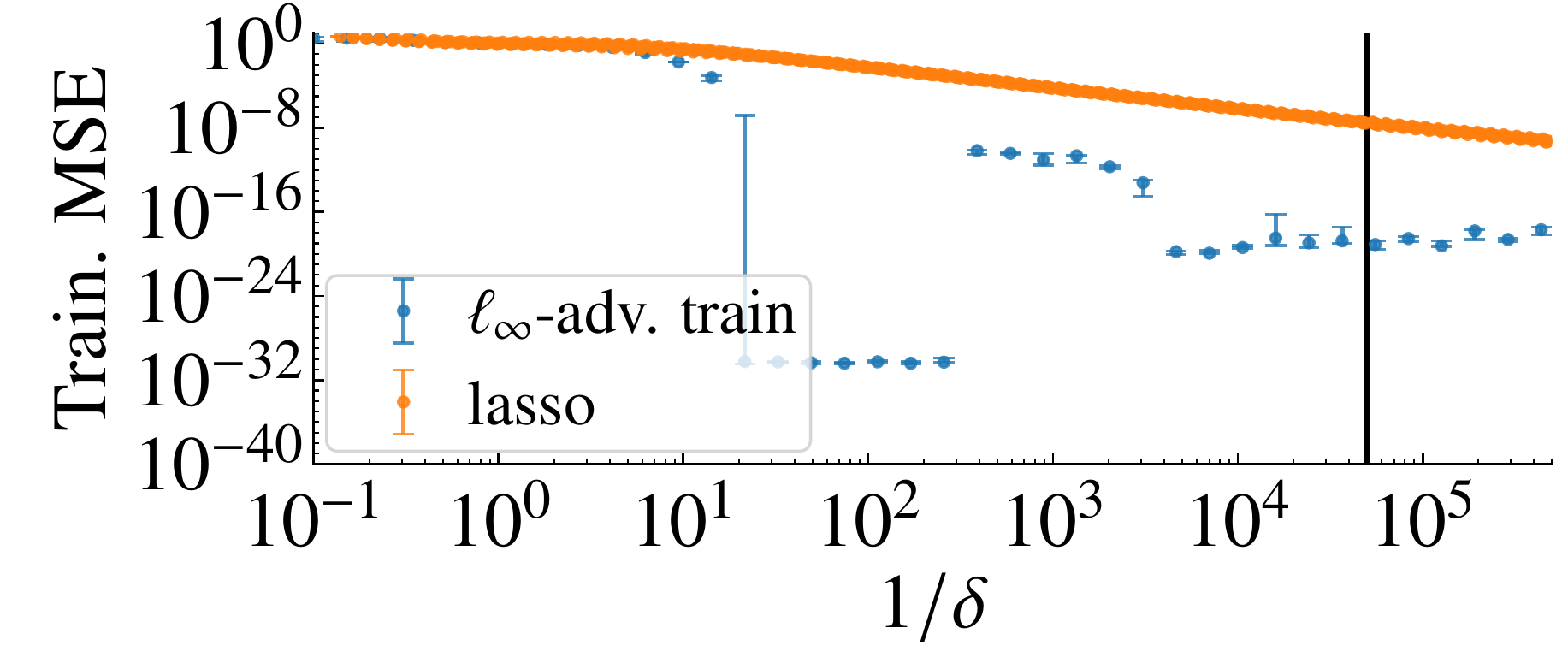}
    \caption{\textbf{Training MSE \textit{vs} regularization parameter (Latent model)}. \emph{Left:} for ridge and $\ell_2$-adversarial training. \emph{Right:}  for lasso and $\ell_\infty$-adversarial training The error bars give the median and the 0.25 and 0.75 quantiles obtained from numerical experiment (5 realizations). Vertical black lines give $\gamma_{\min{}} (\X)$ and, on the figure to the left,  $\gamma_{\min} (\mm{Q}\X)$. I.e., the thresholds obtained from Propositions~\ref{sec:minimizer-lp-attacks} and~\ref{sec:minimizer-l2-attacks}.}
    \label{fig:distance-min-norm-sol-latent}
\end{figure}

\begin{figure}
    \centering
    \subfloat[Ridge regression]{\includegraphics[width=0.4\textwidth]{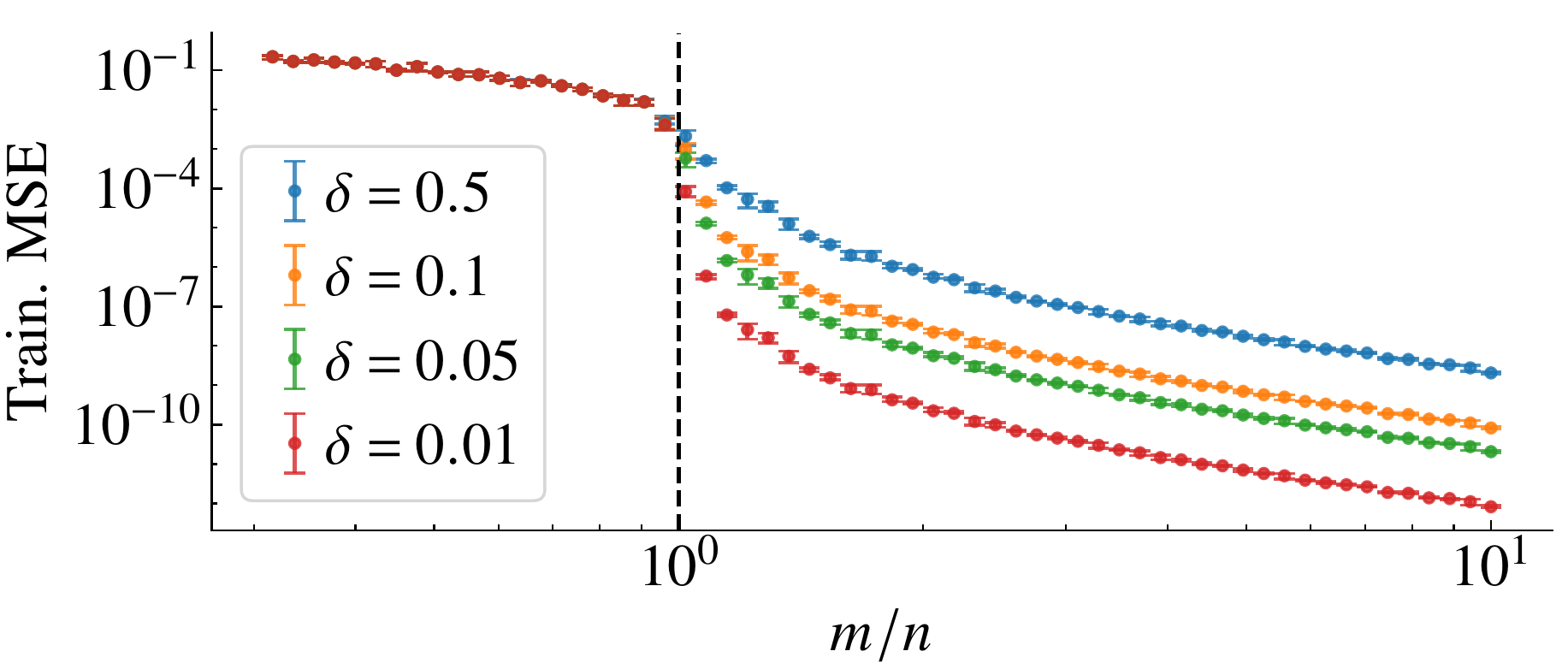}}
    \subfloat[Adversarial training $\ell_2$ ]{\includegraphics[width=0.4\textwidth]{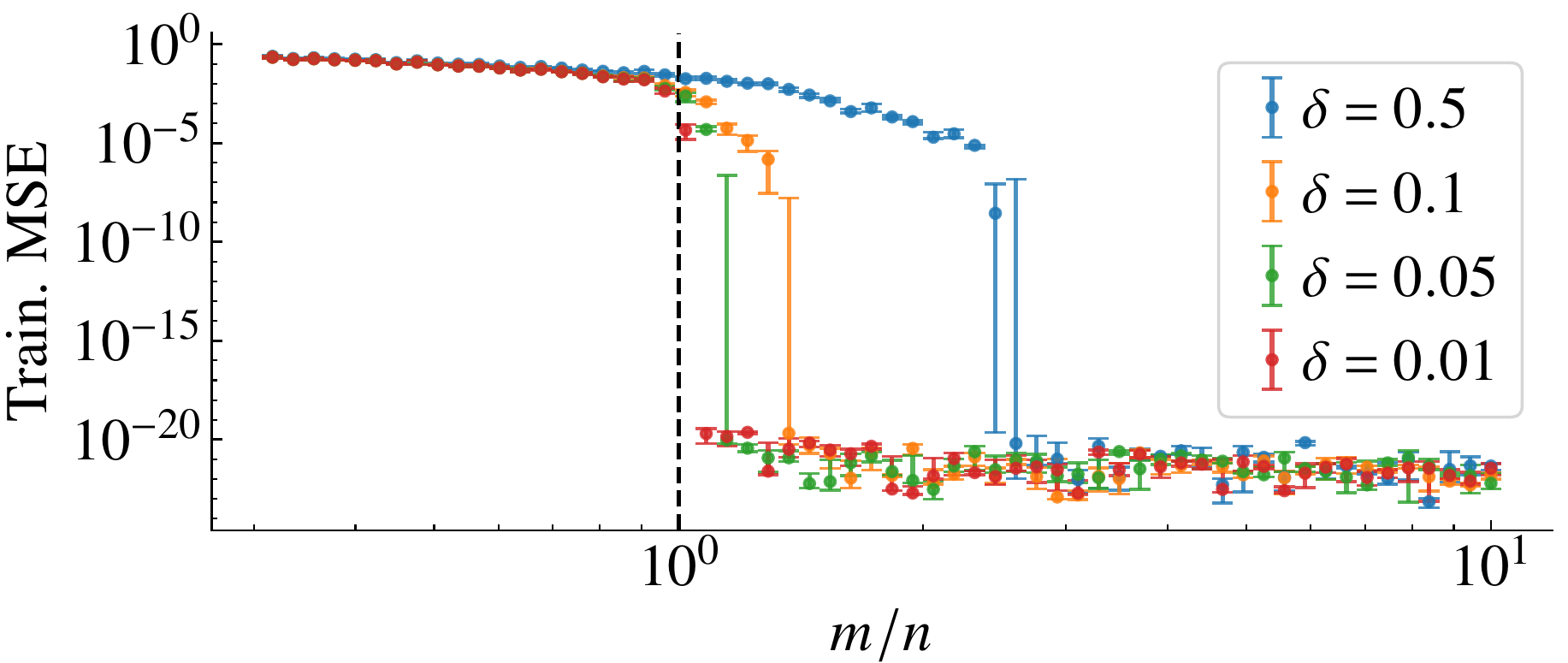}}\\
    \subfloat[Lasso regression]{\includegraphics[width=0.4\textwidth]{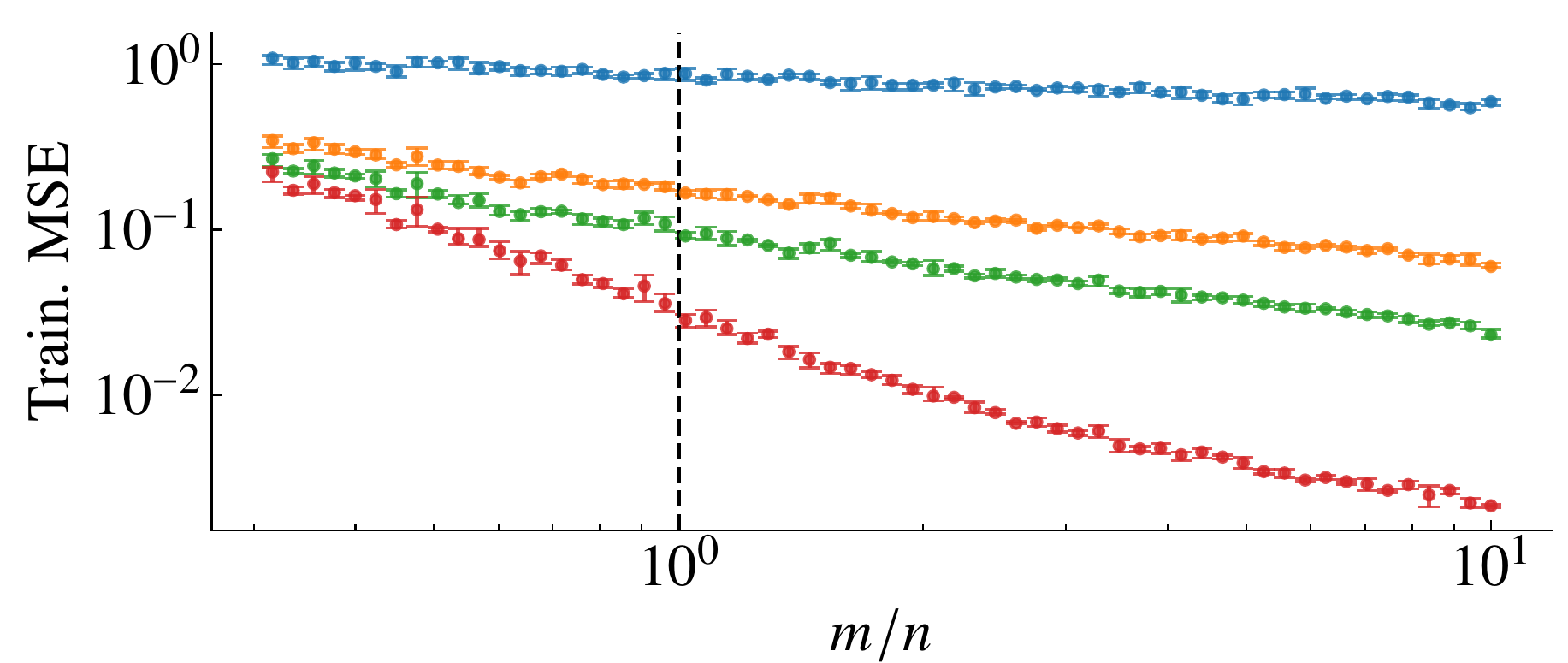}}
    \subfloat[Adversarial training  $\ell_\infty$  ]{\includegraphics[width=0.4\textwidth]{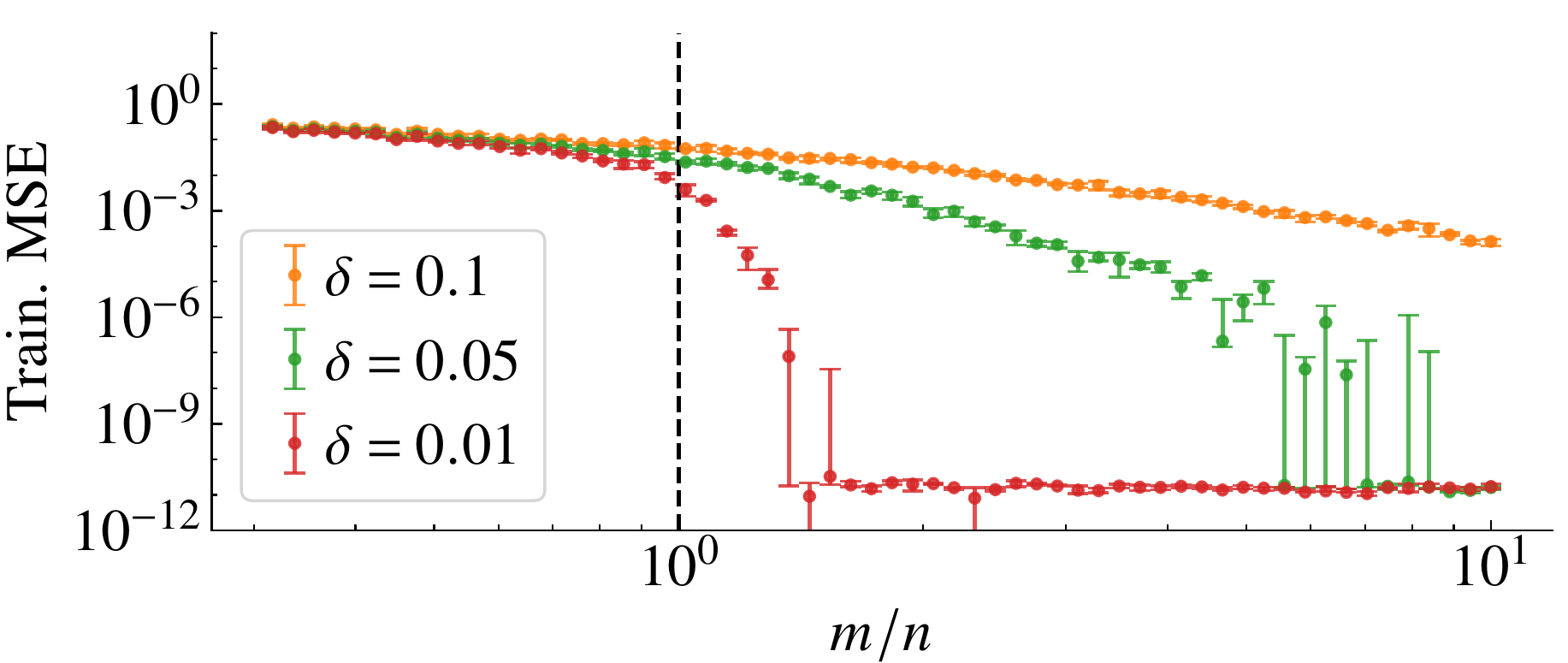}}\\
    \caption{\textbf{Mean square error in training data: latent features}. The plot is equivalent to Fig.~\ref{fig:training-error-advtraining} but with a different data generation procedure. The error bars give the median and the 0.25 and 0.75 quantiles obtained from numerical experiments (6 realizations) and for regularization parameter $\delta = \{0.5,0.1,0.05,0.01\}$ with the colors corresponding to each case indicated in the legend. In (d), we omit the plot for $\delta=0.5$, that is because for this high regularization the solver becomes ill-conditioned.}
    \label{fig:training-error-advtraining-latent}
  \end{figure}

\section{Additional plots: test MSE and parameter norm}

Here we include additional plots related to the experiments, Fig.~\ref{fig:norm-advtraining} show the parameter $\ell_2$-norm and Fig.~\ref{fig:risk-advtraining} show the performance on a hold-out test set for experiments in the isotropic generated data.  Fig.~\ref{fig:norm-advtraining-latent} and Fig.~\ref{fig:risk-advtraining-latent} show the same quantities for experiments in data generated with the latent model. In these plots, we keep $\delta$ and  $n$ constant and vary the number of features $m$.

\begin{figure}[ht]
    \centering
    \subfloat[Ridge regression]{\includegraphics[width=0.4\textwidth]{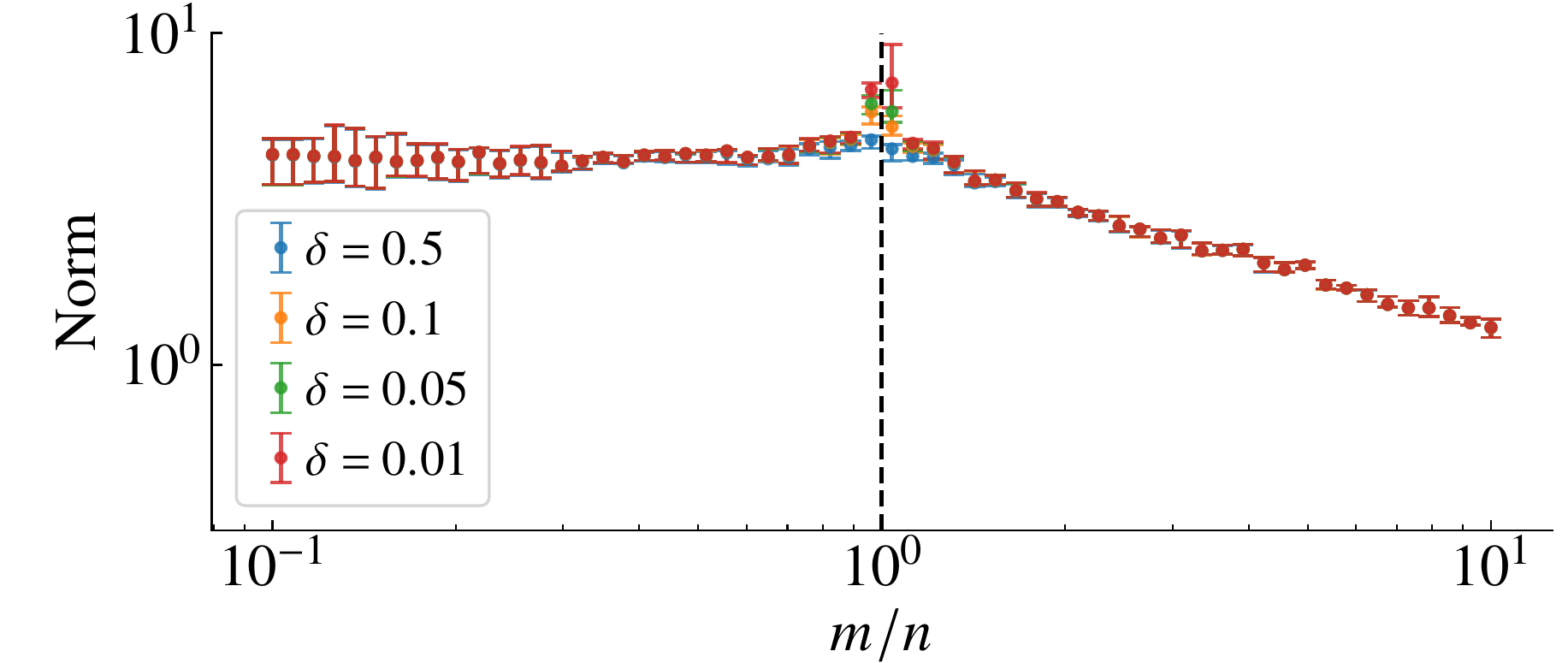}}
    \subfloat[Adversarial training $\ell_2$ ]{\includegraphics[width=0.4\textwidth]{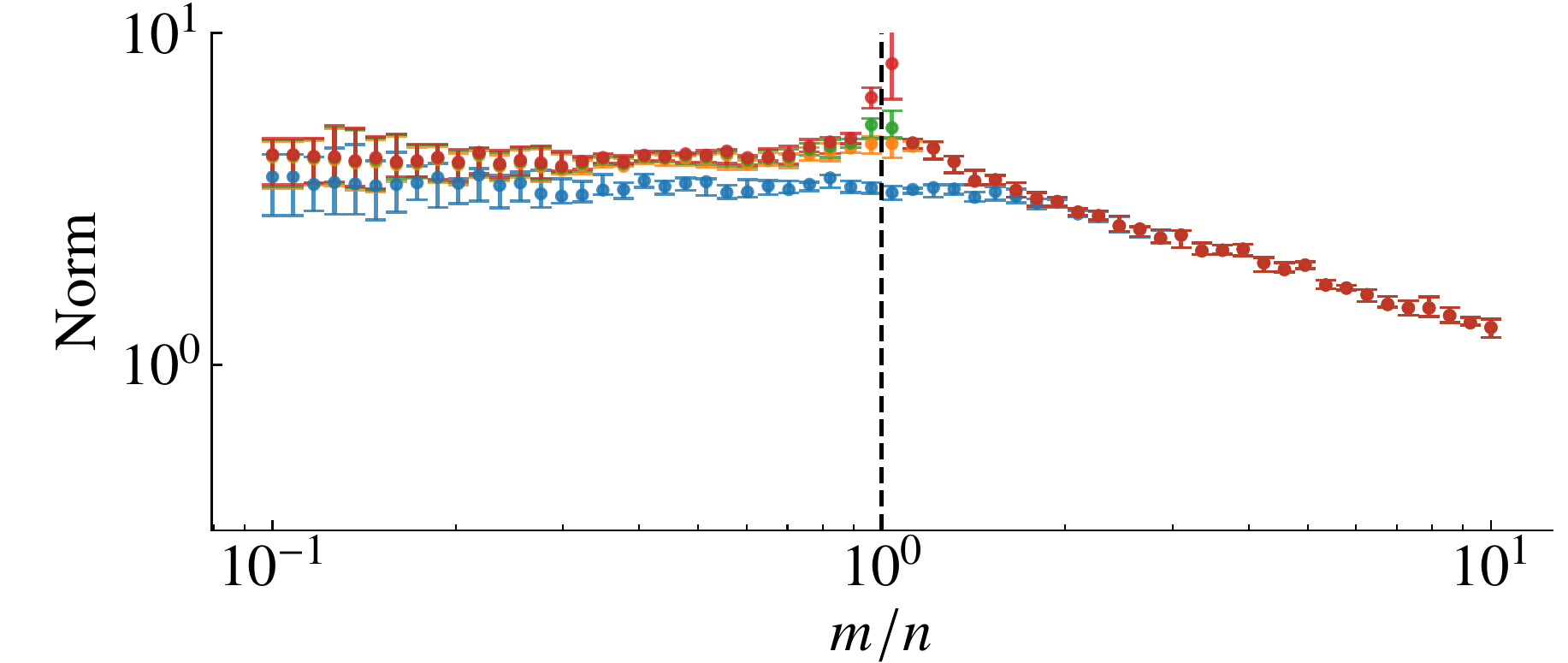}}\\
    \subfloat[Lasso regression]{\includegraphics[width=0.4\textwidth]{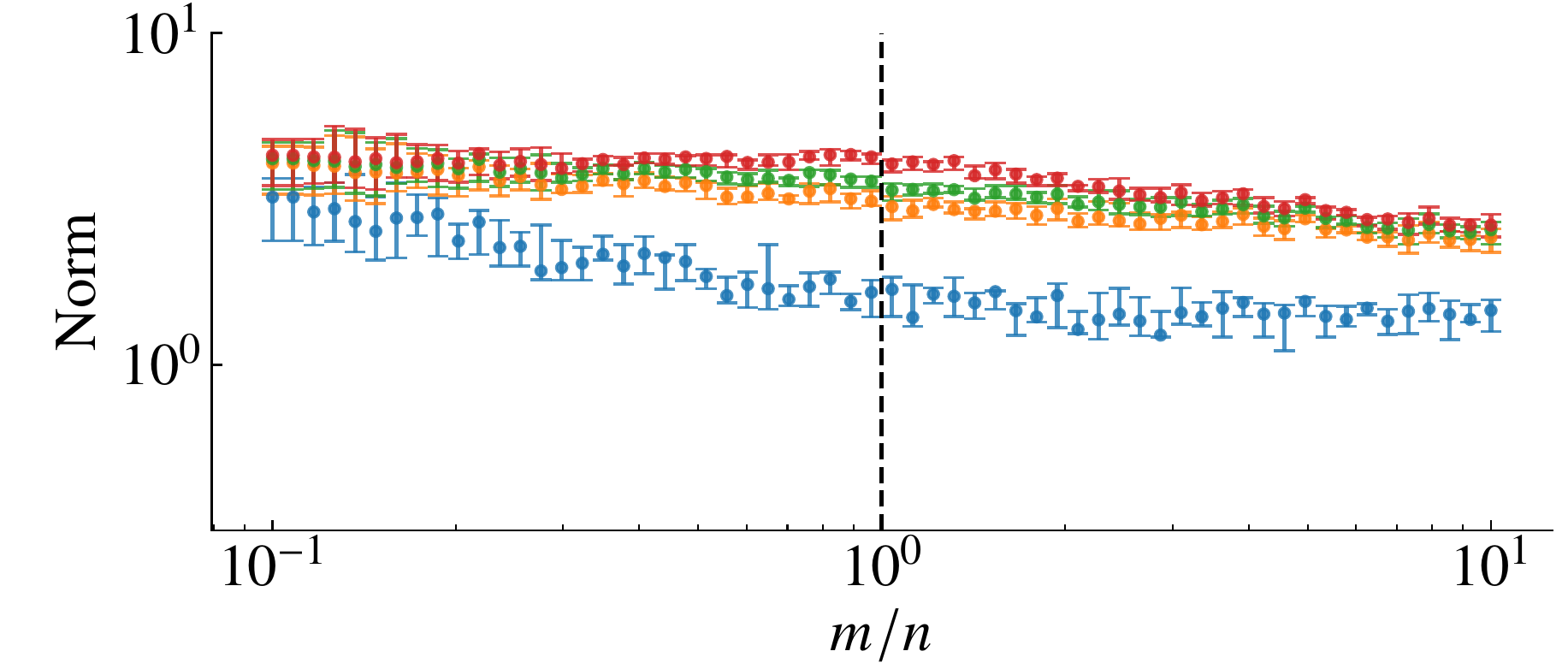}}
    \subfloat[Adversarial training  $\ell_\infty$  ]{\includegraphics[width=0.4\textwidth]{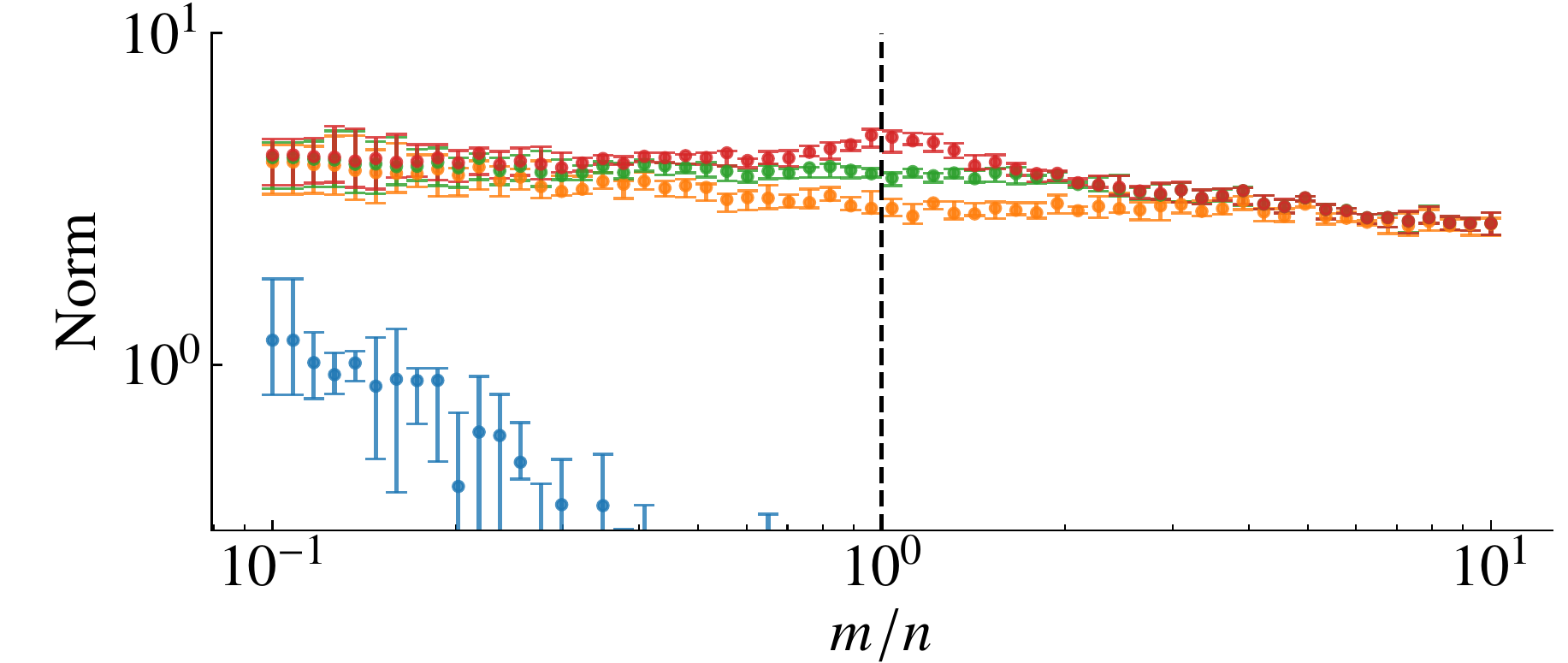}}\\
    \caption{\textbf{Parameter $\ell_2$-norm: isotropic case}. We plot $\|\widehat{\beta}\|_2$ for the parameters estimated by the different methods. The setup is exactly the same as Fig.~\ref{fig:training-error-advtraining}. In (d), the norm for $\delta=0.5$ becomes very close to zero  $\approx10^{-10}$, we let it leave the region in the plot so it is possible to better visualize the curves for the other regularization parameters.} 
    \label{fig:norm-advtraining}
\end{figure}

  \begin{figure}
    \centering
    \subfloat[Ridge regression]{\includegraphics[width=0.4\textwidth]{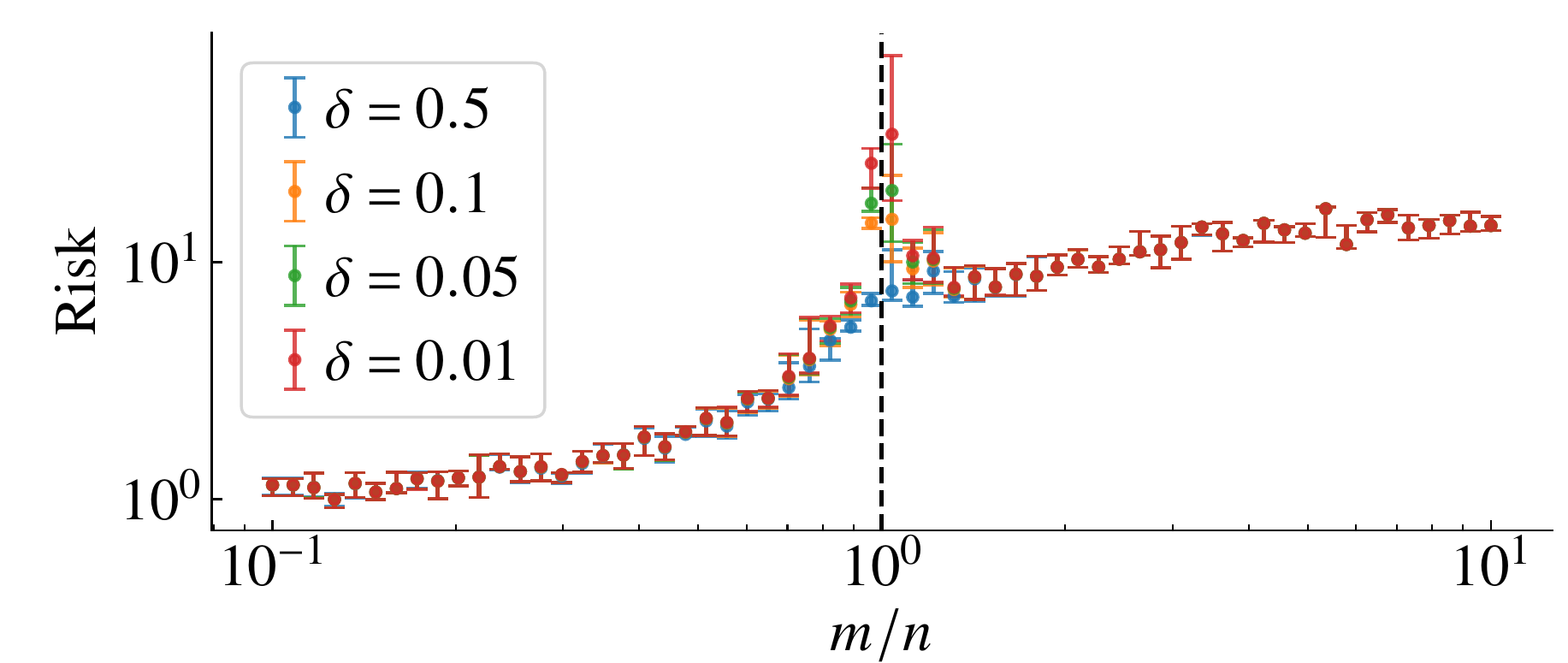}}
    \subfloat[Adversarial training $\ell_2$ ]{\includegraphics[width=0.4\textwidth]{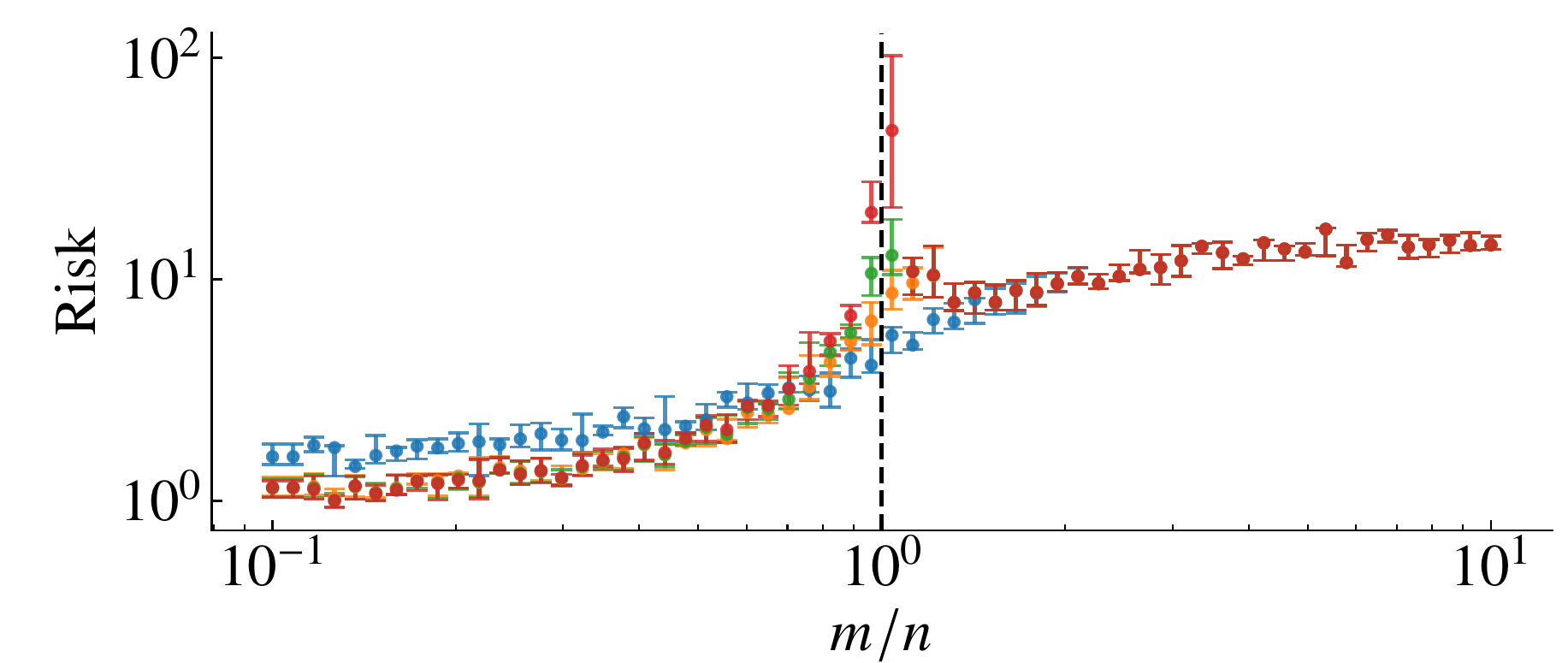}}\\
    \subfloat[Lasso regression]{\includegraphics[width=0.4\textwidth]{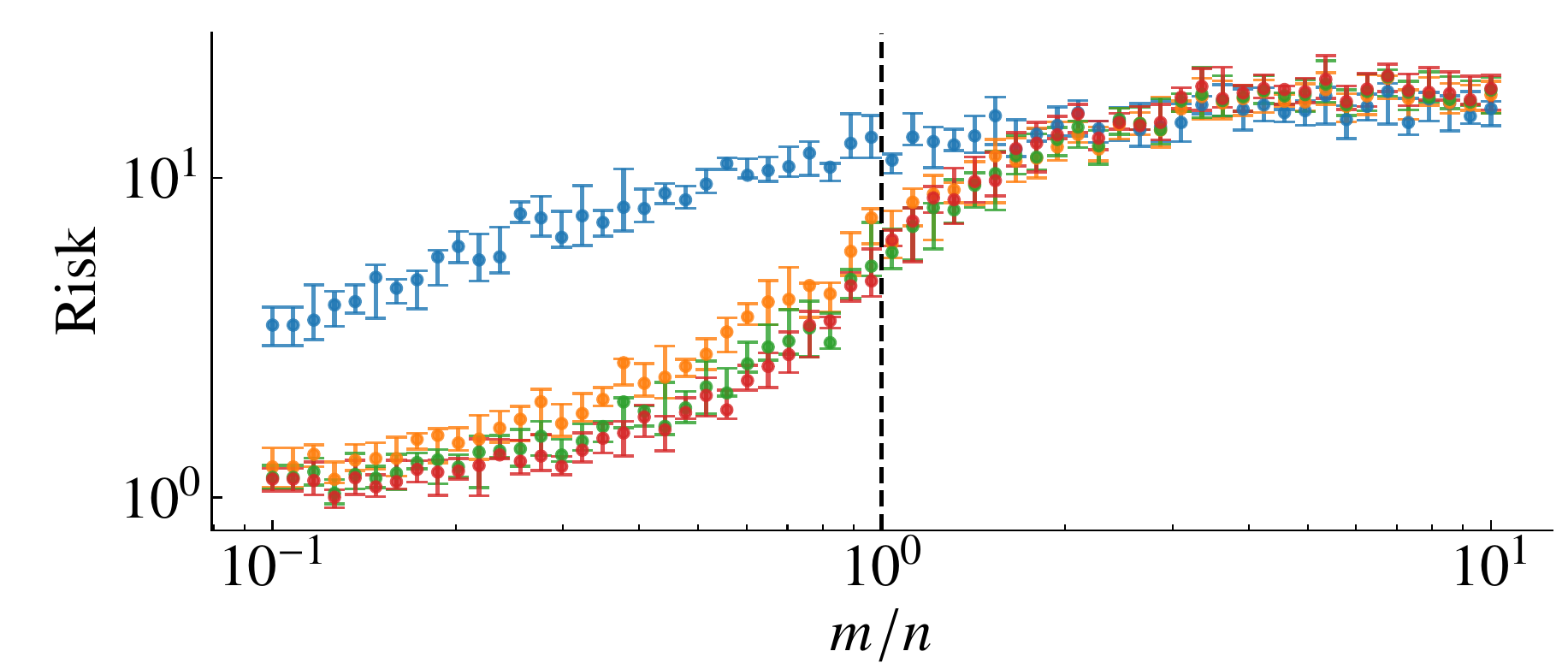}}
    \subfloat[Adversarial training  $\ell_\infty$  ]{\includegraphics[width=0.4\textwidth]{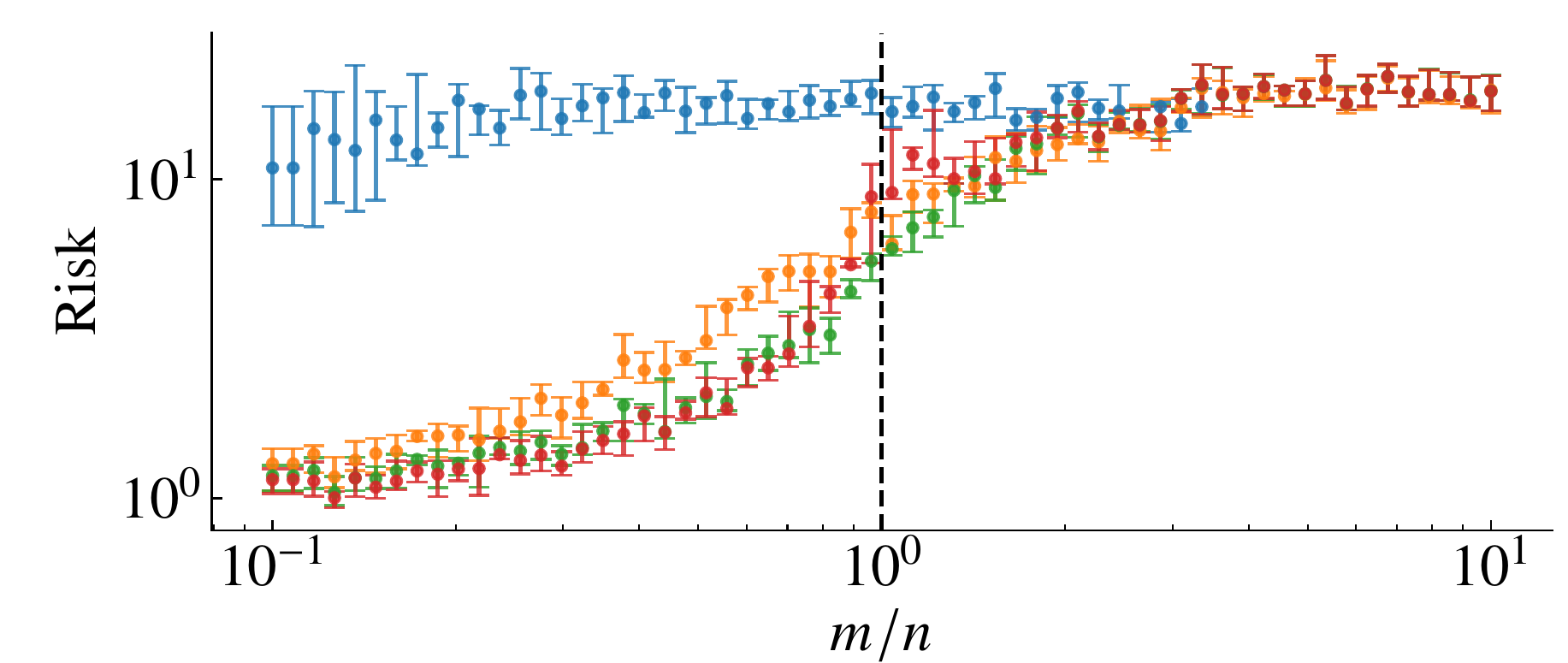}}\\
    \caption{\textbf{Test MSE: isotropic case}. We plot the MSE in a test dataset of size $n_\text{test} = 100$ for the parameters estimated by the different methods. The setup is exactly the same as in  Fig.~\ref{fig:training-error-advtraining}.}
    \label{fig:risk-advtraining}
\end{figure}

\begin{figure}
    \centering
    \subfloat[Ridge regression]{\includegraphics[width=0.4\textwidth]{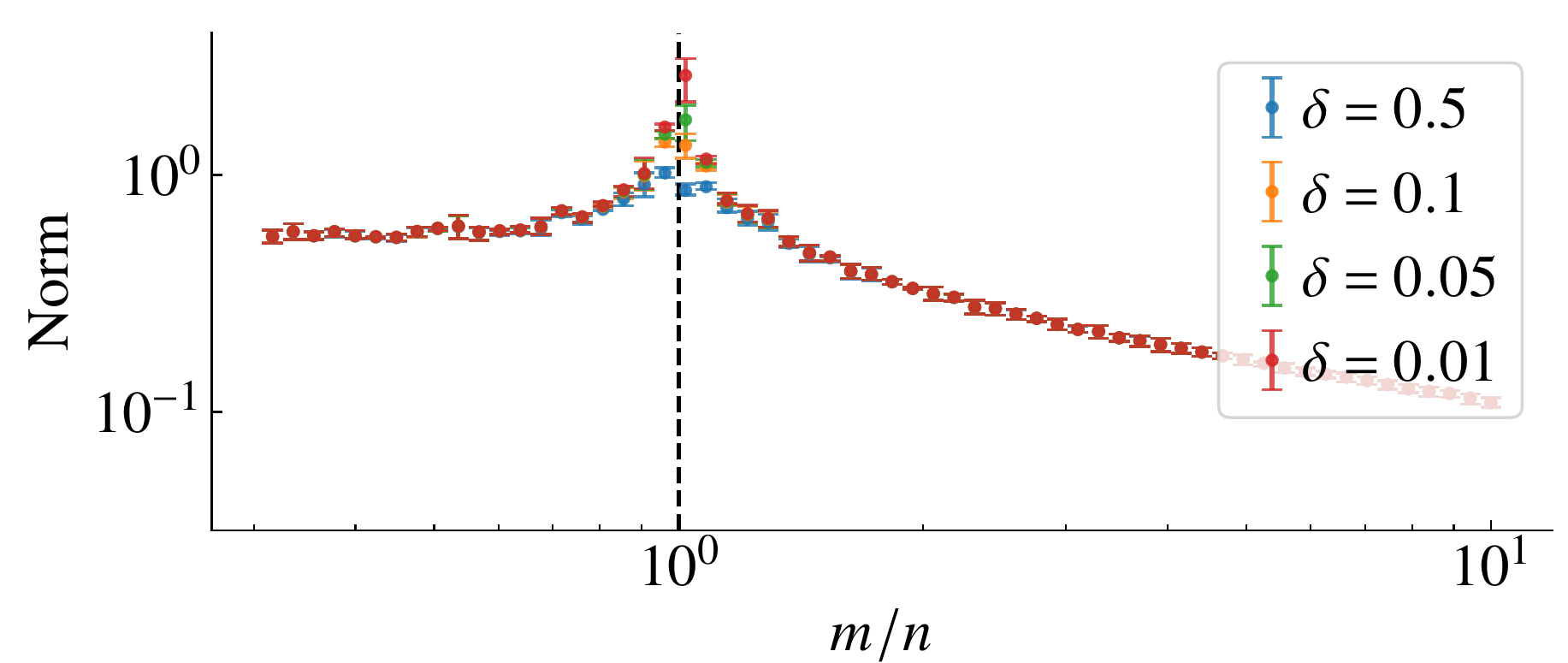}}
    \subfloat[Adversarial training $\ell_2$ ]{\includegraphics[width=0.4\textwidth]{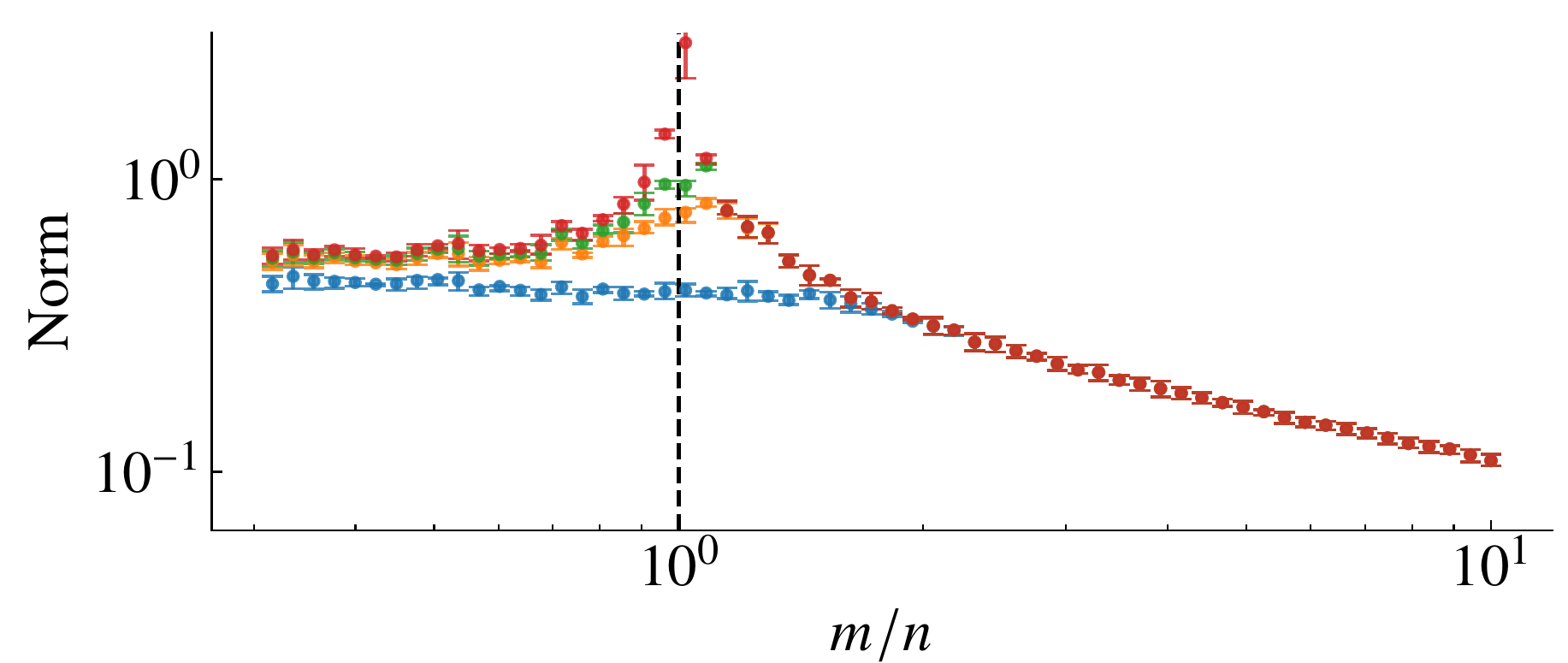}}\\
    \subfloat[Lasso regression]{\includegraphics[width=0.4\textwidth]{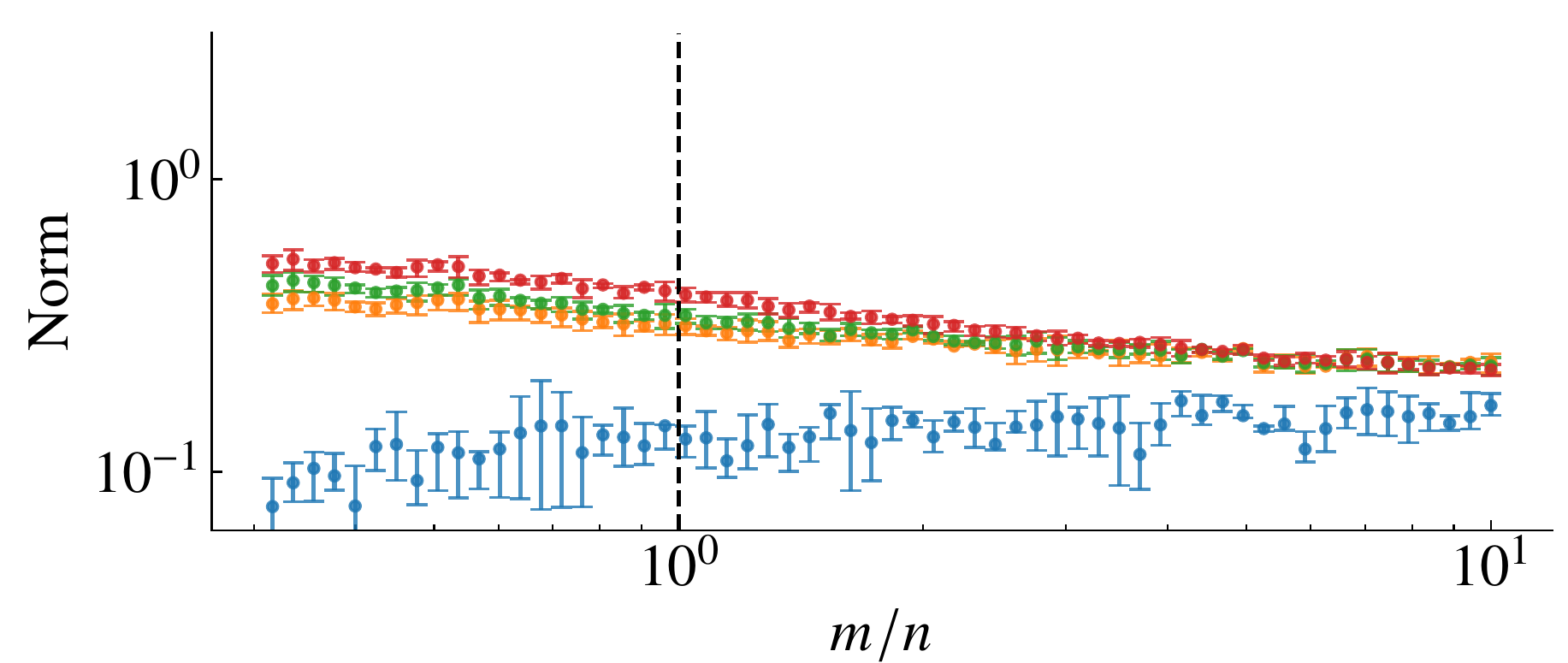}}
    \subfloat[Adversarial training  $\ell_\infty$  ]{\includegraphics[width=0.4\textwidth]{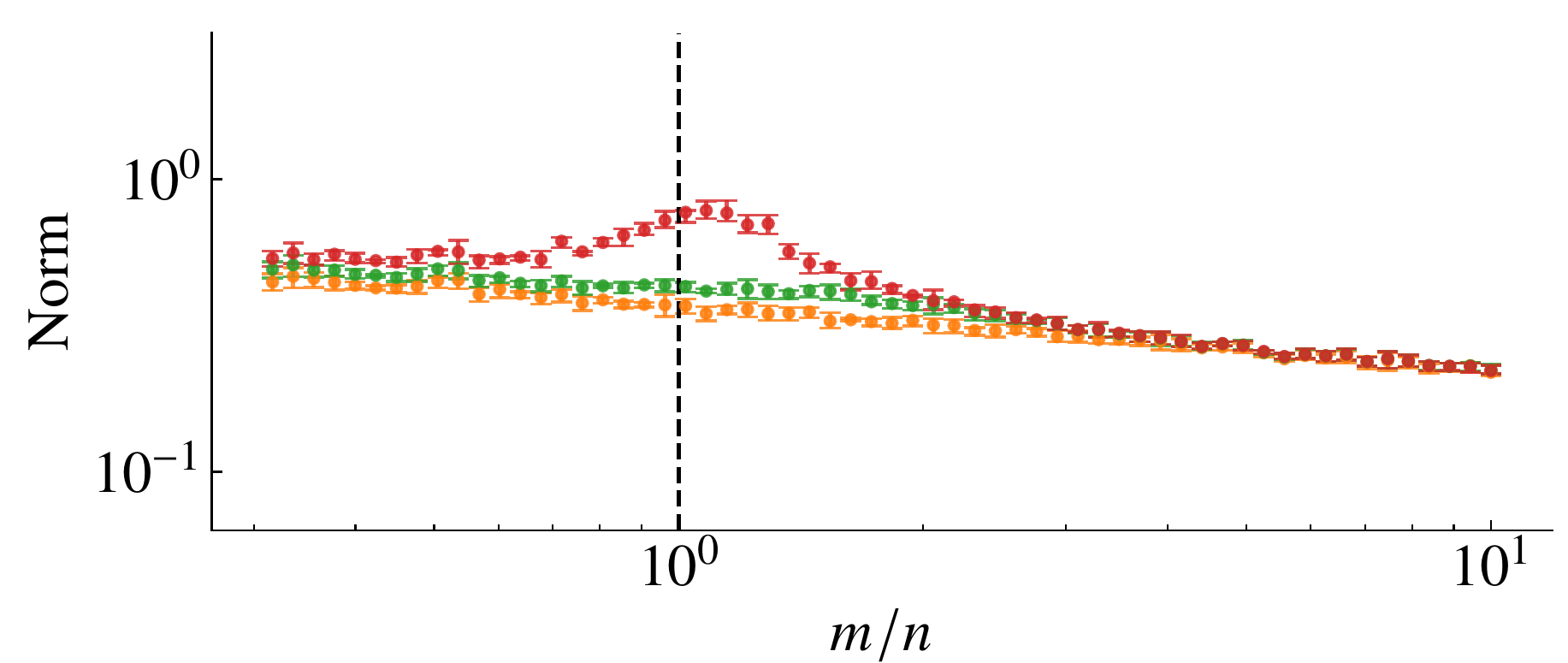}}\\
    \caption{\textbf{Parameter $\ell_2$-norm: latent feature model}. We plot $\|\widehat{\beta}\|_2$ for the parameters estimated by the different methods. The setup is exactly the same as Fig.~\ref{fig:training-error-advtraining-latent}.}
    \label{fig:norm-advtraining-latent}
\end{figure}

  \begin{figure}
    \centering
    \subfloat[Ridge regression]{\includegraphics[width=0.4\textwidth]{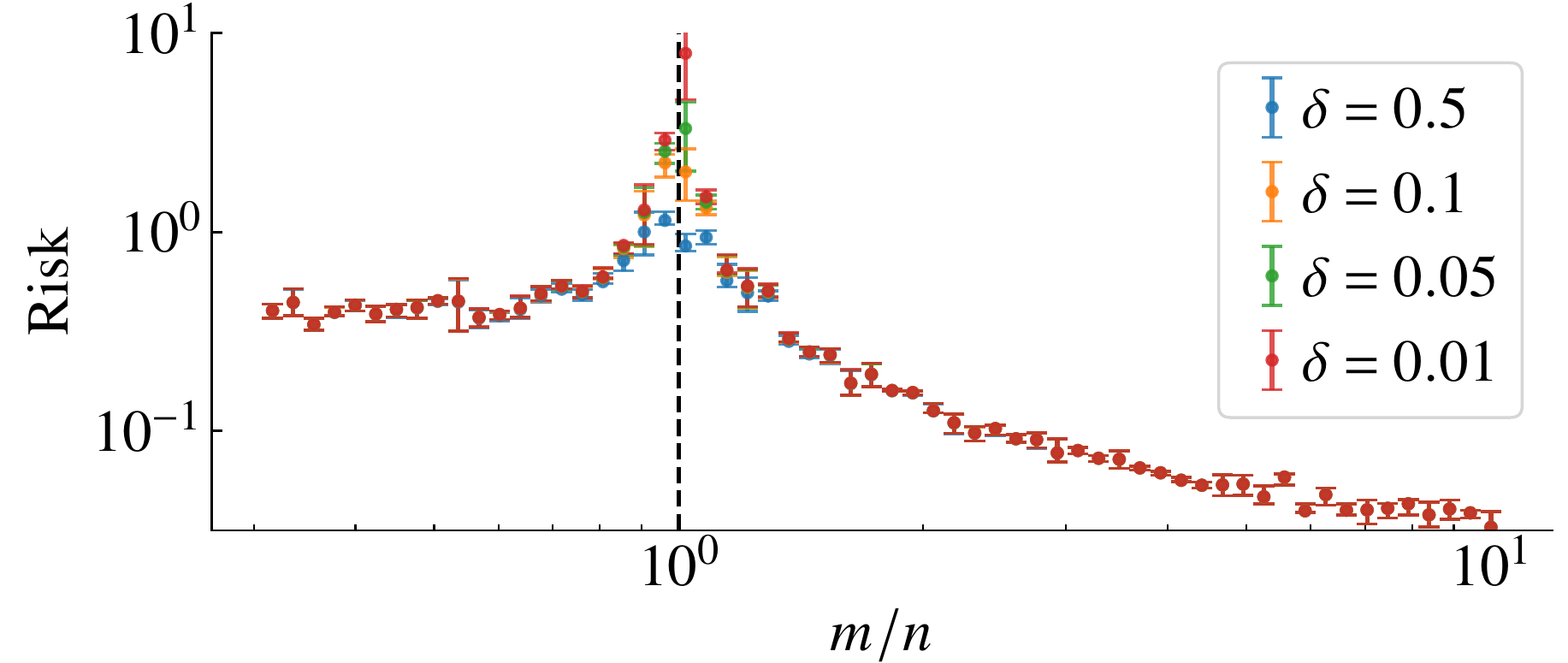}}
    \subfloat[Adversarial training $\ell_2$ ]{\includegraphics[width=0.4\textwidth]{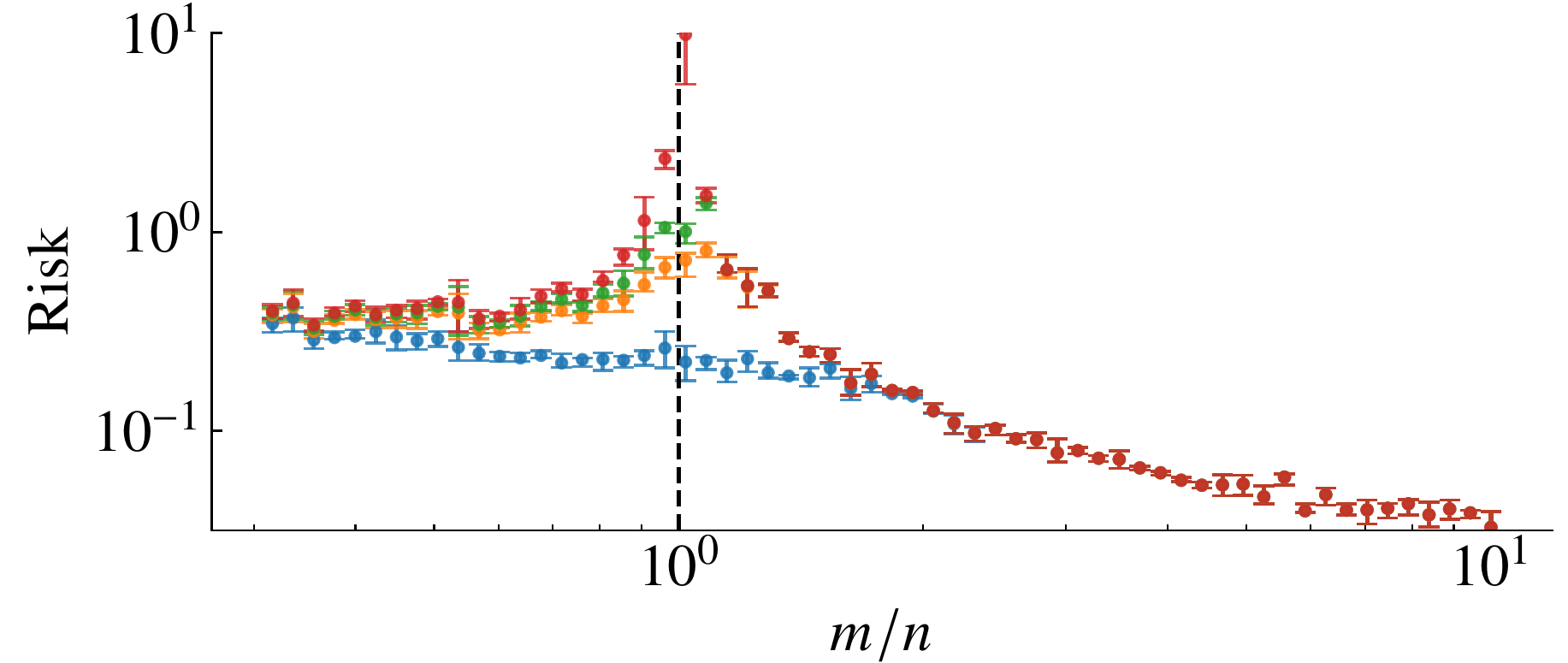}}\\
    \subfloat[Lasso regression]{\includegraphics[width=0.4\textwidth]{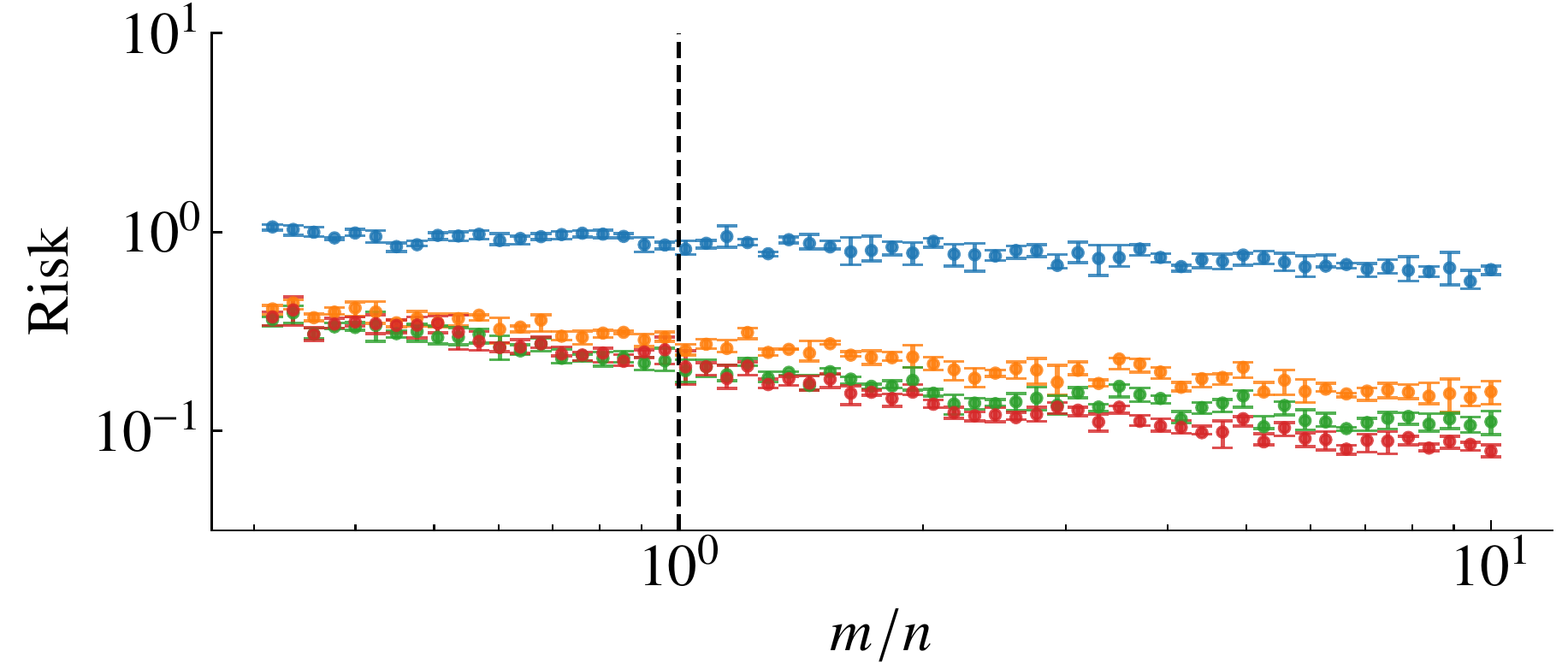}}
    \subfloat[Adversarial training  $\ell_\infty$  ]{\includegraphics[width=0.4\textwidth]{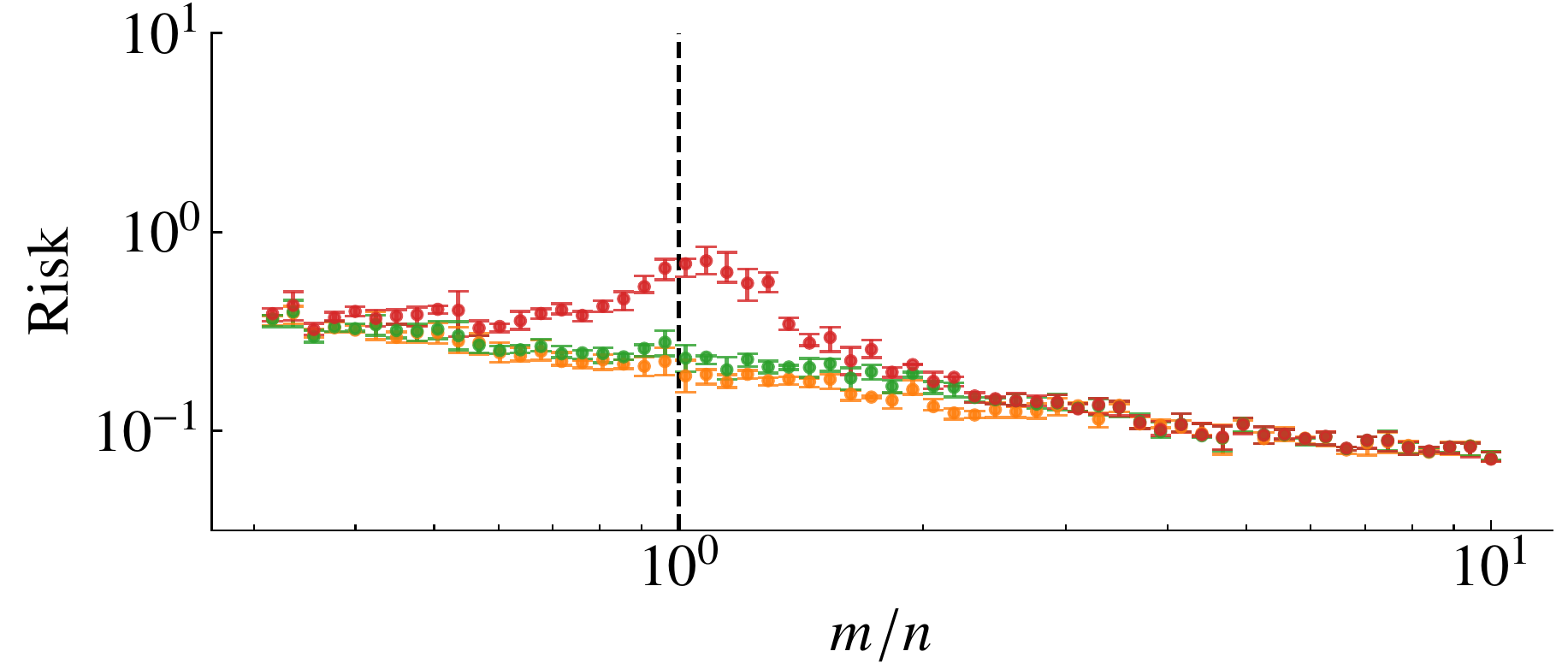}}\\
    \caption{\textbf{Test MSE:  latent feature model}. We plot the MSE in a test dataset of size $n_\text{test} = 100$ for the parameters estimated by the different methods. The setup is exactly the same as Fig.~\ref{fig:training-error-advtraining-latent}.}
    \label{fig:risk-advtraining-latent}
\end{figure}

\end{document}